\documentclass{article}

\usepackage{natbib}
\usepackage{graphicx}
\usepackage{amsthm}
\usepackage{amsmath}
\usepackage{amssymb}
\usepackage{subfigure}
\usepackage{optidef}

\usepackage[utf8]{inputenc} % allow utf-8 input
\usepackage[T1]{fontenc}    % use 8-bit T1 fonts
\usepackage{hyperref}       % hyperlinks
\usepackage{url}            % simple URL typesetting
\usepackage{booktabs}       % professional-quality tables
\usepackage{amsfonts}       % blackboard math symbols
\usepackage{nicefrac}       % compact symbols for 1/2, etc.
\usepackage{microtype}      % microtypography
\usepackage{xcolor}         % colors
\usepackage{enumitem}

\newcommand{\transp}[1]{#1^\top}

\newcommand{\R}{\mathbb{R}}
\newcommand{\N}{\mathbb{N}}

\newcommand{\bigo}{\mathcal{O}}

\DeclareMathOperator*{\argmin}{argmin}
\DeclareMathOperator{\prox}{prox}

\newcommand{\Pro}{\mathbb{P}}

\newcommand{\E}{\mathbb{E}}

\newcommand{\eps}{\varepsilon}

\newcommand{\Obs}[1]{\mathcal{O}_{#1}}

\newcommand{\gau}[1]{\mathcal{N}(0, #1)}

\newcommand{\proxi}[2]{\operatorname{prox}_{#1}{\left(#2\right)} }

\newcommand{\divalpha}[2]{D_\alpha\big(#1\,||\, #2\big)}

\newcommand{\operator}{R}
\newcommand{\bx}{u}
\newcommand{\expect}{\mathbb{E}}

\newcommand{\step}{k}

\newcommand{\blocks}{B}
\newcommand{\block}{b}

\newcommand{\cU}{\mathcal{U}}
\newcommand{\cD}{\mathcal{D}}

\usepackage{mathtools}
\usepackage[linesnumbered, ruled]{algorithm2e}
\usepackage[capitalize,noabbrev]{cleveref}

\DeclarePairedDelimiter\norm{\lVert}{\rVert}
\newtheorem{thm}{Theorem}

\newtheorem{defn}{Definition}
\newtheorem{prop}{Proposition}
\newtheorem{rmk}{Remark}

\newtheorem{assum}{Assumption}
\makeatletter
\newtheorem*{rep@theorem}{\rep@title}
\newcommand{\newreptheorem}[2]{%
	\newenvironment{rep#1}[1]{%
		\def\rep@title{\textbf{#2} \ref{##1}}%
		\begin{rep@theorem}}%
		{\end{rep@theorem}}}
\makeatother
\newreptheorem{theorem}{Theorem}
\newreptheorem{lemma}{Lemma}
\newreptheorem{corollary}{Corollary}
\usepackage{microtype}
\usepackage{graphicx}
\usepackage{subfigure}
\usepackage{booktabs} % for professional tables

\usepackage{hyperref}

\usepackage[accepted]{icml2023}

\usepackage{amsmath}
\usepackage{amssymb}
\usepackage{mathtools}
\usepackage{amsthm}

\usepackage[capitalize,noabbrev]{cleveref}

\theoremstyle{plain}
\newtheorem{theorem}{Theorem}[section]

\newtheorem{lemma}[theorem]{Lemma}
\newtheorem{corollary}[theorem]{Corollary}
\theoremstyle{definition}

\theoremstyle{remark}

\allowdisplaybreaks
\usepackage[textsize=tiny]{todonotes}

\icmltitlerunning{From Noisy Fixed-Point Iterations to Private ADMM for
Centralized and Federated Learning}

\begin{document}

\twocolumn[
\icmltitle{From Noisy Fixed-Point Iterations to Private ADMM\\ for Centralized
and Federated Learning}

% It is OKAY to include author information, even for blind
% submissions: the style file will automatically remove it for you
% unless you've provided the [accepted] option to the icml2023
% package.

% List of affiliations: The first argument should be a (short)
% identifier you will use later to specify author affiliations
% Academic affiliations should list Department, University, City, Region, Country
% Industry affiliations should list Company, City, Region, Country

% You can specify symbols, otherwise they are numbered in order.
% Ideally, you should not use this facility. Affiliations will be numbered
% in order of appearance and this is the preferred way.
\icmlsetsymbol{equal}{*}

\begin{icmlauthorlist}
\icmlauthor{Edwige Cyffers}{inria}
\icmlauthor{Aurelien Bellet}{inria}
\icmlauthor{Debabrota Basu}{inria}
%\icmlauthor{}{sch}
%\icmlauthor{}{sch}
\end{icmlauthorlist}

\icmlaffiliation{inria}{Univ. Lille, Inria, CNRS, Centrale Lille, UMR 9189 - CRIStAL, F-59000 Lille}

\icmlcorrespondingauthor{Edwige Cyffers}{edwige.cyffers@inria.fr}

% You may provide any keywords that you
% find helpful for describing your paper; these are used to populate
% the "keywords" metadata in the PDF but will not be shown in the document
\icmlkeywords{differential privacy, ADMM, federated learning, decentralization}

\vskip 0.3in
]

% this must go after the closing bracket ] following \twocolumn[ ...

% This command actually creates the footnote in the first column
% listing the affiliations and the copyright notice.
% The command takes one argument, which is text to display at the start of the footnote.
% The \icmlEqualContribution command is standard text for equal contribution.
% Remove it (just {}) if you do not need this facility.

\printAffiliationsAndNotice{}  % leave blank if no need to mention equal contribution
% \printAffiliationsAndNotice{\icmlEqualContribution} % otherwise use the standard text.

\begin{abstract}
    We study differentially private (DP) machine learning
    algorithms as instances of noisy fixed-point iterations, in order to
    derive privacy and utility results from this well-studied framework. We
    show that this new perspective recovers popular private gradient-based
    methods like DP-SGD and provides a principled way to design and
    analyze new
    private optimization algorithms in a flexible manner. Focusing on the
    widely-used Alternating Directions Method of Multipliers (ADMM) method, we
    use our general framework to derive novel private ADMM algorithms for
    centralized, federated and fully decentralized learning. For these three algorithms, we establish
    strong privacy guarantees leveraging privacy
    amplification by iteration and by subsampling.
    % We prove
    % strong DP guarantees by relying on 
    % We also introduce a
    % federated ADMM algorithm compatible with
    % client subsampling, which simultaneously provides local DP guarantees and
    % stronger bounds for central DP.
    Finally, we provide utility guarantees using a unified analysis that exploits a recent
    linear convergence result for noisy fixed-point iterations.
\end{abstract}

% !TEX root = main.tex

\section{Introduction}

% \begin{itemize}
%     \item Broadening the scope of privacy amplification by Iteration \cite{ampbyiteration}
%     \item ADMM from a very generic formulation that includes GD and maybe SGD
%     \item ADMM as well-studied, generic, efficient, fast, non smooth and so on \cite{NEURIPS2019_4559912e}
% \end{itemize}

Controlling the risk of privacy leakage in machine learning training and outputs has become of paramount importance in
applications involving personal or confidential data.
This has drawn significant attention to the design of Empirical Risk
Minimization (ERM) algorithms that satisfy Differential Privacy (DP) 
\cite{chaudhuri2011Differentially}. DP is the standard for
measuring the privacy leakage of data-dependent
computations. The most popular approaches to private ERM are Differentially Private Stochastic Gradient Descent (DP-SGD) 
\citep{bassily2014Private,abadi2016Deep} and its variants 
\citep{talwar2015Nearly,wang2017Differentially,zhou2021Bypassing,2021arXiv211011688M,kairouz2021Nearly,pmlr-v139-kairouz21b}. DP-SGD is a first-order optimization algorithm, where the gradients of empirical risks are perturbed with Gaussian noise. Algorithms like DP-SGD can be naturally extended from the classic centralized setting, where a single trusted curator holds the raw data, to federated and decentralized scenarios that involve multiple agents who do not want to share their local data~\citep{dp_fed_sgd_user_level,dp_fedavg_user_level,Noble2022a,Cyffers2020PrivacyAB}.

In this work, we revisit private ERM from the perspective of 
\emph{fixed-point iterations} \cite{10.5555/2028633}, which compute fixed
points of a function by iteratively applying a non-expansive operator $T$.
Fixed point iterations are well-studied and widely applied in mathematical
optimization, automatic control, and signal processing.
They provide a unifying framework that encompasses many optimization
algorithms, from (proximal) gradient descent algorithms to the Alternating
Direction Method of Multipliers (ADMM), and come with a rich theory 
\cite{fixedpointDS}.
Specifically, we study a general \emph{noisy} fixed-point
iteration, where Gaussian noise is added to the operator $T$ at each step. We
also consider a (possibly randomized) block-coordinate version, where the
operator is applied only to a subset of coordinates. As particular cases of
our framework, we show that we can recover DP-SGD and a recent
coordinate-wise variant \cite{2021arXiv211011688M}. We then prove a utility
bound for the iterates of our general framework by exploiting recent linear
convergence results from the fixed-point literature 
\cite{combettes2019stochastic}.

\looseness=-1 With this general framework and results in place, we show that
we can design
and analyze new private algorithms for ERM in a principled manner. We focus on
ADMM-type algorithms, which are known for their effectiveness in centralized
and decentralized machine learning 
\citep{admmbook,Wei2012a,Wei2013a,iutzeler2013,Shi2014c,Vanhaesebrouck2017a,tavara2022federated,fl_admm}.
Based on a reformulation of
ERM as a consensus problem and the characterization of the ADMM iteration as a
Lions-Mercier operator on post-infimal composition \citep{Giselsson2016LineSF}, we derive
private ADMM algorithms for centralized, federated, and fully decentralized
learning. In contrast to previously proposed private ADMM algorithms that
build upon a duality interpretation and
require \emph{ad-hoc}
algorithmic modifications and customized theoretical analysis 
\citep{Huang2020DPADMMAD,7563366,Zhang2018ImprovingTP,Ding2020TowardsPD,Laplacian}, our
algorithms and utility guarantees follow directly from our analysis of our
general noisy fixed-point iteration.
In particular, we are the first to our knowledge to derive a general
convergence rate analysis of private ADMM that can be used for the
centralized, federated, and fully decentralized settings.
We prove strong DP guarantees for our private ADMM
algorithms by properly binding appropriate privacy amplification schemes
compatible with the three settings, such as privacy amplification by iteration
\cite{ampbyiteration} and by subsampling
\citep{amp_sub_mironov}, with a general sensitivity analysis of our
fixed-point formulation.
We believe our findings will serve as a generic and
interpretable recipe to analyze future private
optimization algorithms.

% It extends the use of privacy
% amplification by iteration technique from DP-SGD to private ADMM algorithms.
% to derive an interesting
% privacy-utility trade-off in each of the three regimes.

% !TEX root = main.tex

\section{Related Work}

% Fixed point theory has a longstanding history and has been actively
% studied in the fields of mathematical optimization, automatic control and
% signal processing.

% Fixed point algorithms? Long history, provides a unifying framework, popular
% in mathematical optimization, signal processing and automatic control, not
% much taken in machine learning which focused on the study of particular
% cases like (stochastic) gradient descent and ADMM.

The ERM framework is widely used to efficiently train machine learning models
with DP. Here, we briefly review the approaches based on privacy-preserving
optimization, which are closest to our work and popular in practice due to
their wide applicability.\footnote{Other techniques such as output and
objective perturbation have also been considered, see
e.g.~\citep{chaudhuri2011Differentially}.}

\paragraph{Private gradient-based methods.} Differentially Private Stochastic
Gradient Descent (DP-SGD) \citep{bassily2014Private,abadi2016Deep} and
its numerous variants 
\citep{talwar2015Nearly,wang2017Differentially,zhou2021Bypassing,2021arXiv211011688M,kairouz2021Nearly,pmlr-v139-kairouz21b} are
extensively studied and deployed for preserving DP while training ML models. Since these methods interact with data through the computation of gradients, DP is ensured by adding calibrated Gaussian
noise to the gradients.
% As shown in Section~\ref{TODO}, our framework recovers
% DP-SGD algorithms as a special case.
% In the centralized setting under record-level
% DP, the analysis of DP-SGD typically requires the use of privacy amplification
% by subsampling \citep{} to obtain utility guarantees that improve with the
% number of data points. % simplify? we don't care
These approaches naturally extend to federated learning 
\citep{kairouz_advances_2019},
where several users (clients) aim to collaboratively train a model without
revealing their local dataset (user-level DP). In particular, DP-FedSGD 
\citep{dp_fed_sgd_user_level}, DP-FedAvg \citep{dp_fedavg_user_level} and
DP-Scaffold \citep{Noble2022a} are federated extensions of DP-SGD that rely on
an (untrusted) server to aggregate the gradients or model updates from (a
subsample of) the users.
% \footnote{DP-FedAvg
% implements multiple local updates. Note that our federated and
% decentralized algorithms naturally implement multiple local updates as the
% local step consists in solving a local subproblem.}
These algorithms provide a \emph{local} DP guarantee with respect to the
server (who observes individual user contributions), and a stronger 
\emph{central} DP guarantee with respect to a third party observing only the
final model.\footnote{The stronger central DP guarantees holds also w.r.t. the
server if secure aggregation is used \citep{Bonawitz2017a}.} In the
fully decentralized setting, the server is replaced by direct
user-to-user communications along the edges of a communication graph.
\citet{Cyffers2020PrivacyAB} and \citet{Cyffers2022b} recently showed that
fully decentralized variants of DP-SGD provide stronger privacy guarantees
than suggested by a local DP analysis. Their
results are based on the notion of network DP, a relaxation of local
DP capturing the fact that users only observe the information they
receive from their neighbors in the communication graph.

% The privacy-utility trade-off (i.e., the amount of noise that needs to be
% added) depends on the trust model.
% federated: client level-DP, secure aggregation --> we get this for our ADMM
% algorithms.
% decentralized: LDP only. Cyffers & Bellet show amplification is possible for
% DP-SGD under appropriate relaxation when random-walk --> we extend such
% results to ADMM algorithms.
% main message should be: with our fixed point iteration view we are able to
% port this to ADMM algorithms

As we will show in Section~\ref{sec:fix}, our general noisy fixed-point
iteration
framework allows recovering these private gradient-based algorithms as a special
case, but also to derive novel private ADMM algorithms (in the centralized,
federated and fully decentralized settings) with privacy and utility
guarantees similar to their gradient-based counterparts.
% Private version of other
% gradient-based methods, notably
% Frank-Wolfe algorithms for solving constrained problems, have also been
% studied \cite{talwar2015Nearly,asi2021Private,bassily2021NonEuclidean}.
% However, we are not aware of prior work which takes the general perspective of
% noisy fixed point iterations to design and analyze differentially private
% optimization algorithms.

% Private gradient descent: DP-SGD
% in centralized setting, analyzed via subsampling or iteration.
% DP-SGD extends naturally to federated and decentralized scenarios, with LDP. Typically want to consider client-level DP; multiple local updates (DP-FedAvg - note that we naturally do multiple local updates); amplification difficult to apply; resort to secure aggregation
% In decentralized setting, often LDP (poor utility). Recent work on amplification by iteration. Our work extends this idea to ADMM.

% Private ERM: mostly gradient descent algorithms (which we generalize by our
% framework). Although some other algorithms have been explored (e.g., FW),
% we are not aware of approaches taking the perspective of noisy fixed
% point iterations.

% Privacy amplification by iteration: based on contraction, naturally fits
% fixed point, but exploited only for particular case of gradient descent.

\paragraph{Private ADMM.} \looseness=-1  Due to the flexibility and
effectiveness of ADMM for
centralized and decentralized machine learning
\citep{admmbook,Wei2012a,Shi2014c,Vanhaesebrouck2017a},
differentially private versions of ADMM have been studied for the centralized 
\citep{Xu2021DifferentiallyPA,Laplacian}, federated
\citep{Huang2020DPADMMAD,9188006,2022arXiv220209409R, hu2019learning}, and fully decentralized 
\citep{7563366,Zhang2018ImprovingTP,Ding2020TowardsPD} settings.
These existing private ADMM
algorithms are specifically crafted for one of the three settings, based on
ad-hoc algorithmic modifications and customized analysis that are
not extendable to the other settings.
For example, the previous fully decentralized private ADMM algorithms use at
least 3-4
privacy parameters and add noise from two different distributions 
\cite{7563366,Ding2020TowardsPD}, while the centralized private ADMM of
\citet{Xu2021DifferentiallyPA} uses one noise generating distribution.
Thus, it is very hard to find an
overarching generic structure in the previous literature. Our work simplifies
and unifies the design of private ADMM algorithms by developing a generic
framework: we provide a unified utility analysis, and the same
baseline privacy analysis based on sensitivity with a clear
parametrization by a single parameter, for the three settings (centralized,
federated and fully decentralized).
We achieve this thanks to our characterization of private ADMM algorithms as
noisy fixed-point iterations, and more specifically as noisy Lions-Mercier
operators on post-infimal composition \citep{Giselsson2016LineSF} rather than
on the dual functions.
In contrast, previous work on private ADMM mostly used a dual function
viewpoint, leading
to complex convergence analysis (sometimes with restrictive
assumptions) and privacy
guarantees that are difficult to interpret (and often limited to LDP).
% In particular, the privacy guarantees usually rely on perturbation of the
% primal variables, often through noise addition to the first-order approximation of
% the corresponding subproblem, similarly to DP-SGD \cite{7563366,
% 9188006,plausible,Xu2021DifferentiallyPA}. This leads to complex convergence
% analysis with potentially restrictive assumptions, and results in privacy
% guarantees that are difficult to interpret and are often limited to LDP. In
% contrast, our framework leads to the addition of noise to the dual variable
% and the privacy analysis requires no additional hypothesis compared to
% standard ADMM. Furthermore, our centralized, federated and fully
% decentralized algorithms and their analysis all naturally follow from
% our generic (block-wise) noisy fixed point iteration formulation.
We also note that, except for \citet{9188006}
who considered only the trusted server setting, we are the
first to achieve user-level DP for federated and fully decentralized ADMM.
As discussed below, we are also the first to show that ADMM can
benefit from privacy amplification to obtain better privacy-utility
trade-offs.
% Private ADMM: work in various settings (centralized,
% federated/distributed, decentralized over a graph) - smooth / nonsmooth,
% strongly convex only...
% Requires ad-hoc algorithmic modifications and convergence/utility analysis;
% in contrast, our algorithms and analysis naturally follow from 
% (coordinate) noisy fixed point method and lead to centralize/federated/decentralized
% We are the first to consider client-level DP in federated/decentralized
% setting (except [4] with trusted server)
% We are the first to prove some privacy amplification for ADMM (in
% centralized + random-walk settings)

\paragraph{Privacy amplification by iteration.} The seminal work of 
\citet{ampbyiteration}, later extended by \citet{Altschuler2022}, showed that
iteratively
applying non-expansive updates can amplify privacy guarantees for data points
used in early stages. % (ref. Theorem~\ref{thm:amp_iter} in Section~
% \ref{sec:dp}).
Although privacy amplification by
iteration is quite
general, to the best of our knowledge, it was successfully applied only to
DP-SGD. In this work, we show how to leverage it for the ADMM algorithms applied to the consensus-based problems in the fully decentralized setting.

More generally, \textit{our work stands out as we are not aware of any prior
work that considers the general perspective of noisy-fixed point iterations to
design and analyze differentially private optimization algorithms.}

% !TEX root = main.tex

\section{Background}
\label{sec:back}
In this section, we introduce the necessary background that will constitute
the basis of our contributions. We start by providing basic intuitions and
results about the fixed-point iterations framework. Then, we show how
ADMM fits into this framework. Finally, we introduce Differential
Privacy (DP) and the technical tools used in our privacy
analysis.

\subsection{Fixed-Point Iterations}

\looseness=-1 Let us consider the problem of finding a minimizer (or
generally, a stationary point) of a function $f:\cU\rightarrow
\mathbb{R}$, where $\cU\subseteq \R^p$. This problem reduces to
finding a point $u^*\in\cU$ such that $0\in\partial f(u^*)$, or
$\nabla f(u^*)=0$, when $f$
is differentiable. A generic approach to compute $u^*$ is to iteratively apply
an operator $T:\cU\rightarrow\cU$ such that the fixed points
of $T$, i.e., the points $u^*$ satisfying $T(u^*)=u^*$, coincide with the
stationary points of $f$. The iterative application of $T$ starting from an initial point $u_0\in\cU$ constitutes the \emph{fixed-point iteration} framework \cite{10.5555/2028633}:
\begin{equation}
\label{eq:fixed_point_basic}
u_{k+1} \triangleq T(u_k).
\end{equation}
We denote by $I$ the identity operator, i.e. $I(u) \triangleq u$. To analyze the
convergence of the sequence of iterates to a
fixed point of $T$, various assumptions on $T$ are considered.

\begin{defn}[Non-expansive, contractive, and $\lambda$-averaged operators]
\label{def:op}
Let $T:
\cU\rightarrow\cU$ and $\lambda \in (0,1)$. We say that:
\begin{itemize}[leftmargin=20pt,topsep=0pt,itemsep=0ex]
\item $T$ is \emph{non-expansive} if it is 1-Lipschitz, i.e., $\|T(u)-T
(u')\|\leq \|u - u'\|$ for all $u,u'\in\cU$.
\item $T$ is \emph{$\tau$-contractive} if it $\tau$-Lipschitz with $\tau<1$.
\item $T$ is \emph{$\lambda$-averaged} if there
    exists a non-expansive operator $R$ such that $T = \lambda R + (1 - \lambda) I$.
\end{itemize}
\end{defn}
Hereafter, we will focus on $\lambda$-averaged operators that correspond
to a barycenter between the identity mapping and a non-expansive operator.
This family encompasses many popular optimization algorithms. For instance,
when $f$ is convex and $\beta$-smooth, the operator $T=I -\gamma\nabla f$,
which corresponds to gradient descent, is $\gamma \beta/2$-averaged for
$\gamma\in(0,2/\beta)$. The proximal point, proximal gradient and ADMM
algorithms also belong to this family \cite{10.5555/2028633}.
By the Krasnosel'skii Mann theorem \cite{Byrne_2003}, the iterates of a
$\lambda$-averaged operator converge. Hence, \textit{formulating an optimization algorithm as the application of a $\lambda$-averaged operator allows us to reuse generic convergence results.}

\looseness=-1 The rich convergence theory of fixed point iterations goes well
beyond the
simple iteration \eqref{eq:fixed_point_basic}, see \cite{fixedpointDS} for
a recent overview. In this work, we leverage several extensions of this theory. First, we consider \emph{inexact updates}, where each application of $T$ is perturbed
by additive noise of bounded magnitude. Such noise can arise because the
operator is computed only approximately (for higher efficiency) or
due to the stochasticity in data-dependent computations. Another extension considers $T$
operating on a decomposable space $\cU =
\cU_1 \times \dots \times \cU_B$ with $B$ blocks, i.e.,
\begin{equation*}
   T(u) \triangleq (T_1(u), \dots, T_B(u)),\quad \text{where }
   T_b:\cU\rightarrow\cU_b, \forall b.
\end{equation*}
Here, it is possible to \emph{update each block
separately} in order to reduce per-iteration
computational costs and memory requirements, or to facilitate
decentralization \cite{walkmanMao2020}. This corresponds to replacing the
update in
\eqref{eq:fixed_point_basic} by:
\begin{equation}
    \forall b:~ u_{k+1,b} = u_{k,b} + \rho_{k,b} (T_b
    (u_k) - u_{k,b}),
\end{equation}
where $\rho_{k,b}$ is a Boolean (random) variable that encodes if block
$b$ is updated at iteration $k$.\footnote{Note that these block updates can be seen
as projections of the global update and thus are also non-expansive.}
Block-wise fixed-point iterations have first been introduced in
\cite{iutzeler2013,bianchi2016}. Various
strategies for selecting blocks are possible, such as cyclic updates or
random sampling schemes.
% increase flexibility of implementation and eventually
% decentralization as seen in Section~\ref{sec:admm}. A specific case of the separation is when blocks correspond to the different coordinates, which has been proven useful in machine learning and more specifically for preserving privacy \cite{https://doi.org/10.48550/arxiv.2110.11688}.
A generic convergence analysis of fixed-point iterations under both inexact
and block updates has been proposed by \citet{combettes2019stochastic}, which
we leverage in our analysis.

\subsection{ADMM as a Fixed-Point Iteration}
\label{sec:admm_fixed}

\looseness=-1 We now present how ADMM can be defined as a fixed-point
iteration.
ADMM minimizes the sum of two (possibly non-smooth) convex functions with
linear constraints between the
variables of these functions, which can be formulated as:
\begin{mini}
    {x, z}{\label{eq:genericADMM} f(x)+g(z)}{}{}
    \addConstraint{{Ax+Bz}}{=c}{}
\end{mini}

ADMM is often presented as an approximate version of the augmented Lagrangian
method, where the minimization of the sum in the primal is approximated by the
alternating minimizations on $x$ and $z$. However, this analogy is not
fruitful for theoretical analysis, as no proof of convergence only relies on
bounding this approximation error to analyze ADMM 
\cite{Eckstein2015UnderstandingTC}. A more useful characterization of ADMM is
to see it as a splitting algorithm \cite{Eckstein2015UnderstandingTC}, i.e.,
an approach to find a fixed point of the composition of two (proximal) operators by
performing operations that involve each operator separately. 

Specifically, ADMM can be defined through the
Lions-Mercier operator
\cite{lions}. 
Given two proximable functions $p_1$ and $p_2$ and parameter $\gamma>0$, the
Lions-Mercier operator is:
\begin{equation}
\label{eq:lions}
T_{\gamma p_1, \gamma p_2} = \lambda R_{\gamma p_{1}}R_{\gamma p_{2}} + (1-\lambda) I,
\end{equation}
where $R_{\gamma p_{1}}=2\prox_{\gamma p_1} - I$ and $R_{\gamma p_{2}}=2\prox_
{\gamma p_2} - I$. This operator is $\lambda$-averaged, and it can be shown
that if the set of the zeros of $\partial (f + g)$ is not empty, then the
fixed points of $T_{\gamma p_1, \gamma p_2}$ are exactly these
zeros \cite{admmbook}.

The fixed-point iteration \eqref{eq:fixed_point_basic}
with $T_{\gamma p_1, \gamma p_2}$ is known as the Douglas-Rachford algorithm,
and ADMM is equivalent to this algorithm applied to a reformulation of 
\eqref{eq:genericADMM} as $\min_u p_1(u) + p_2(u)$ with $p_1(u)=(-A
\triangleright f)(-u-c)$ and $p_2(u)=(-B \triangleright g)(u)$, where we
denote by $(M
\triangleright f)(y)=\inf \{f(x) \mid M x=y\}$ the infimal
postcomposition \cite{Giselsson2016LineSF}. For completeness, we show in
Appendix~\ref{app:non-private-admm} how to recover the
standard ADMM updates from this formulation.

% This gives the Douglas Rachford algorithm, with the updates :
% \[z^{k+1} = \prox_{\gamma g}(x^k)
%     \]
% \[x^{k+1} = \prox_{\gamma f}(2 z^{k+1} - x^k)-z^{k+1}+x^k
%     \]
% which are equivalent to $u^{k+1}=T(u^k)$.

% ADMM can be solved by applying this Lions-Mercier operator in the dual space, using infimal postcomposition.  Hence, noticing this equivalence, we have de facto the convergence of the ADMM update towards a local minimizer in the dual space.

\subsection{Differential Privacy}
\label{sec:dp}

In this work, we study fixed-point iterations with Differential Privacy 
(DP), which is the de-facto standard to quantify the
privacy leakage of algorithms \cite{dwork2013Algorithmic}.
DP relies on a notion of neighboring datasets.
We denote a private dataset of size $n$ by $\cD\triangleq(d_1,\dots,d_n)$. Two datasets $\cD, \cD'$ are neighboring if they
differ in at most one element $d_i\neq d_i'$, and we note this relation $\cD\sim \cD'$. We refer to
each $d_i$ as a \emph{data item}. Depending on the context
(centralized
versus federated), $d_i$ corresponds to a data point
(\emph{record-level} DP), or to the whole local dataset of a user
(\emph{user-level} DP).
% , as previously done in \cite{ampbyiteration}.

Formally, we use Rényi Differential Privacy (RDP) 
\cite{DBLP:journals/corr/Mironov17} for its theoretical convenience and better
composition properties. We recall that any $(\alpha, \eps)$-RDP algorithm is
also $
(\eps+\ln
(1/\delta)/(\alpha-1),\delta)$-DP for any $0<\delta<1$ in the classic $
(\eps, \delta)$-DP definition. 

\begin{defn}[Rényi Differential Privacy (RDP)~\cite{renyifounda}]
    \label{def:RDP}
    Given $\alpha>1$ and $\eps>0$, an algorithm $A$ satisfies $(\alpha,
    \eps)$-Rényi Differential Privacy
    if for all pairs of neighboring datasets $\cD \sim \cD'$:
    \begin{equation}
    \label{eq:rdp}
    D_{\alpha} \left(A(\cD) || A(\cD') \right) \leq \eps\,,
    \end{equation}
    where for two random variables $X$ and $Y$, and $\divalpha{X}{Y}$ is the 
    \emph{Rényi divergence} between $X$ and $Y$, i.e.
    \begin{equation*}
        \divalpha{X}{Y}\triangleq \frac{1}{\alpha-1}\ln \int \left(\frac{\mu_{X}(z)}{\mu_Y(z)}  \right)^{\alpha} \mu_Y (z) dz \,,
    \end{equation*}
    with $\mu_X$ and $\mu_Y$ the respective densities of $X$ and $Y$.
  \end{defn}

A standard method to turn a data-dependent computation $h(\cD)\in\R^p$ into an
RDP algorithm is the Gaussian mechanism 
\cite{dwork2013Algorithmic,DBLP:journals/corr/Mironov17}. Gaussian mechanism is defined as
$A(\cD) \triangleq h
(\cD) + \eta$, where $\eta$ is a sample from $\gau{ \sigma^2\mathbb{I}_p}$. This mechanism satisfies $(\alpha, \alpha \Delta^2/2 \sigma^2)$-RDP for
any $\alpha >1$, where $\Delta\triangleq\sup_{\cD\sim \cD'}\norm{h(\cD)-h(\cD')}$ is the
sensitivity of $h$.

\looseness=-1 Our analysis builds upon privacy amplification results (we summarize
the
ones we use in \cref{app:amp}). This includes \emph{amplification by
subsampling} \citep{amp_sub_mironov}: if the above algorithm $A$ is executed on a random
fraction $q$ of $\cD$, then it satisfies
$(\alpha,\mathcal{O}(\alpha \Delta^2q^2/2 \sigma^2))$-RDP.
We also use \emph{privacy amplification by iteration}
\cite{ampbyiteration,Altschuler2022}. This technique
captures the fact that sequentially applying a non-expansive operator
improves privacy guarantees for the initial point as the number
of subsequent updates increase.
% Informally, the privacy leakage occurring at a given update decreases as the trajectories of two different scenarios move closer together over time. 
% \begin{thm}[Privacy amplification by iteration~\cite{ampbyiteration}]
% \label{thm:amp_iter}
%     Let $T_{1}, \dots, T_{K}, T'_{1}, \dots, T'_{K}$ be non-expansive
%     operators, $U_{0}\in\cU$ be an initial random state, and $(\zeta_{k})_
%     {k=1}^K$ be a
%     sequence
%     of noise
%     distributions. Now, consider the noisy iterations $U_{k+1}\triangleq T_{k+1}(U_k)+\eta_
%     {k+1}$ and
%     $U'_{k+1}\triangleq T_{k+1}(U'_k)+\eta'_{k+1}$, where $\eta_{k+1}$ and $\eta_{k+1}'$ are drawn independently from distribution $\zeta_{k+1}$.
%     Let $s_k \triangleq \sup_{u\in\cU} \norm{T_k (u) - T'_k(u)}$. Let $(a_k)_{k=1}^K$ be a
%     sequence of
%     real
%     numbers such that
%     \[\forall k \leq K, \sum_{k' \leq k} s_{k'} \geq \sum_{k' \leq k} a_{k'}, 
%     \text{
%     and } \sum_{k \leq K} s_k = \sum_{k \leq K} a_k\,.
%         \]
%     Then,
%     \begin{equation}
%         D_{\alpha}(U_K || U'_K) \leq \sum_{k=1}^K \sup _{u:\|u\| \leq a_k} D_
%         {\alpha}(\zeta_k * \mathbf{u} \| \zeta_k)\,,
%     \end{equation}  
%     where $*$ is the convolution of probability distributions and $
%     \mathbf{u}$ denotes the distribution of the random variable that is always
%     equal to $u$.
% \end{thm}
 \citet{ampbyiteration} and \citet{Altschuler2022} applied this result to
 ensure
differential privacy for SGD-type algorithms.
\textit{In this work, we use this result in tandem with the generic fixed-point iteration approach to develop and analyze the privacy of ADMM algorithms.}

% !TEX root = main.tex

\section{A General Noisy Fixed-Point Iteration for Privacy Preserving Machine Learning}
\label{sec:fix}

In this section, we formulate privacy preserving machine learning algorithms
as instances of a general noisy fixed-point iteration. We show that we can recover popular private gradient descent methods (such as DP-SGD) from this formulation, and we provide a generic utility analysis.

\subsection{Noisy Fixed-Point Iteration}

\begin{algorithm}[t]
    \caption{Private fixed point iteration}\label{algo:pfix}
    \DontPrintSemicolon
    \SetKwInput{KwInput}{Input}
    \KwInput{Non-expansive operator $R=(R_1,\dots,R_B)$
    over $1\leq B\leq p$
    blocks, initial point
    $u_0\in \cU$, step sizes $(\lambda_k)_{k\in \N}\in(0,1]$, active blocks $
    (\rho_k)_{k\in\N}\in\{0,1\}^B$,
    errors $(e_k)_{k\in \N}$, privacy noise variance
    $\sigma^2\geq 0$}
    \For{$k=0, 1, \dots$}{
    \For{$b=1, \dots, B$}{
        $u_{k+1,b} = u_{k,b} + \rho_{k,b} \lambda_k(R_b(u_k) + e_{k,b} + \eta_
        {k+1,b} - u_{k,b}) \text{ with } \eta_{k+1,b}\sim\gau{ \sigma^2
        \mathbb{I}_p}$\;
    }}
\end{algorithm}

Given a dataset $\cD= (d_1,\dots,d_n)$,
we aim to design differentially private algorithms to approximately solve
the ERM problems of the form:
\begin{mini}
    {u\in\cU\subseteq\mathbb{R}^p}{\label{eq:simple_erm}\frac{1}{n}\sum_
    {i=1}^n f(u;d_i) +
    r(u),}{}
    {}
\end{mini}
where $f(\cdot;d_i)$ is a (typically smooth) loss function computed on data
item $d_i$ and $r$ is a (typically non-smooth) regularizer. We denote $f
(u;\cD)\triangleq \frac{1}{n}\sum_{i=1}^n f(u;d_i)$.

\looseness=-1 To solve this problem, we propose to consider the
general noisy fixed-point iteration described in Algorithm~\ref{algo:pfix}. The core of each update applies a $\lambda_k$-averaged operator constructed from a non-expansive operator $R$,
and a Gaussian noise term added to ensure
differential privacy via the Gaussian mechanism (Section~\ref{sec:dp}).
% Although the Gaussian perturbation is added here for privacy purposes,
% inexact scheme are also widely-studied due to the impossibility to compute operator application or due to the stochasticity in data-dependent computation.
Algorithm~\ref{algo:pfix} can use (possibly randomized) block-wise
updates ($B> 1$) and accommodate additional errors in operator evaluation (in
terms of $e_k$). 

Despite the generality of this scheme, we show in
Section~\ref{sec:generic_conv} that we can provide a unified utility
analysis under the only assumption that the operator $R$ is contractive.
% based only on simple properties of the operator $R$ algorithm  that we will
% only require general properties on $T$ to derive a convergence rate for the
% iterates and for deriving privacy guarantee only thanks to the
% non-expansiveness of $T$.

\subsection{Recovering Private Gradient-based Methods from the Noisy
Fixed-Point Iteration}

Differentially Private Stochastic Gradient Descent (DP-SGD) 
\cite{bassily2014Private,abadi2016Deep} is the most widely used private
optimization algorithm. In Proposition~\ref{prop:dpsgd}, we show that
we recover DP-SGD from our general noisy fixed-point iteration
(Algorithm~\ref{algo:pfix}).

% \aurelien{defined like below, $R$ is not non-expansive; we need to consider
% $R(u) = u - \frac{2}{\beta}\nabla g(\cdot,\cD)$ or $R(u)=\frac{1}{\beta}\nabla g(\cdot,\cD)$ and
% adjust the value of
% $\lambda$ accordingly. Same for Remark 1.}
\begin{prop}[DP-SGD as a noisy fixed-point iteration]
\label{prop:dpsgd}
Assume that $f(\cdot;d)$ is $\beta$-smooth for any $d$, and let $r(u)=0$.
Consider the non-expansive operator
$R(u)\triangleq u - \frac{2}{\beta}\nabla f(u; \cD)$. Set $B=1$, $\lambda_k =
\lambda =\frac{\gamma\beta}{2}$ with $\gamma\in(0,\frac{2}{\beta})$, and $e_k
= \frac{2}{\beta}(\nabla f(u_k) - \nabla f(u_k;d_{i_k}))$ with $i_k\in
\{1,\dots,n\}$.\footnote{One can draw $i_k$ uniformly at random, or choose
it so as to do deterministic passes over $\cD$.}
Then, Algorithm~\ref{algo:pfix} recovers DP-SGD 
\cite{bassily2014Private,abadi2016Deep}, i.e., the update at step $k+1$ is $u_{k+1}=u_k -
\gamma (\nabla f(u_k;d_{i_k}) + \eta'_{k+1})$ with $\eta'_{k+1}\sim\gau{
\frac{\beta^2}{4}\sigma^2I}$.
The term $e_k$ corresponds to the
error due to evaluating the gradient on $d_{i_k}$ only, and
satisfies $\E[\|e_k\|]\leq 4L/\beta$ when $f(\cdot;d)$ is $L$-Lipschitz
for any $d$.\end{prop}

The privacy guarantees of DP-SGD can be derived: first, by observing
that $R(u_k) + e_k = u_k - \nabla f_{i_k}(u_k;d_{i_k})$ is itself
non-expansive, and
then applying privacy amplification by iteration, as done in 
\cite{ampbyiteration}. Alternatively, composition and privacy amplification by
subsampling can be used \cite{amp_sub_mironov}.

Similarly, we also recover Differentially Private Coordinate Descent (DP-CD) 
\citep{2021arXiv211011688M}.

\begin{prop}[DP-CD as a noisy fixed-point iteration]
Consider the same setting as in Proposition~\ref{prop:dpsgd},
but with $B > 1$ blocks (coordinates), and $R_b(u)\triangleq u_b -
\frac{2}{\beta}\nabla_b f
(u; \cD)$, where
$\nabla_b f$ is the $b$-th block of $\nabla f$, and $e_k=0$. Then
Algorithm~\ref{algo:pfix} reduces to the Differentially Private Coordinate Descent (DP-CD) algorithm \cite{2021arXiv211011688M}.
\end{prop}

Utility guarantees for DP-SGD and DP-CD can be obtained as instantiations of
the general convergence analysis of Algorithm~\ref{algo:pfix}, presented in Section~\ref{sec:generic_conv}.

\subsection{Utility Analysis}
\label{sec:generic_conv}

In this section, we derive a utility result for our general noisy fixed-point
iteration when the operator $R$ is \emph{contractive} (see
Definition~\ref{def:op}).
For gradient-based methods, this holds notably when $g$ is smooth and strongly
convex. This is also the case for ADMM \citep[see][and references therein for
contraction constants under various sufficient conditions]
{Giselsson2014,ryu_contraction}.
Our result, stated below, leverages a recent convergence result for inexact
and block-wise fixed-point iterations \citep{combettes2019stochastic}.
Obtaining explicit guarantees for the noisy setting requires a careful
analysis with appropriate upper and lower bounds on the feasible learning
rate, control of the impact of noise, and finally the characterization of the
contraction factor in the convergence rate.
The proof can be found in Appendix~\ref{app:convergence}.

% \aurelien{IMPORTANT: this result (and the original one in 
% \cite{combettes2019stochastic}) only hold for \emph{strictly}
% quasi-nonexpansive operators, i.e., $\tau < 1$. We need to make sure that at
% least the operator used in ADMM satisfies this and clarify under which
% hypotheses on the functions...}

% \aurelien{it holds for gradient descent when $g$ is strongly convex}

% \aurelien{looks like it holds also for Douglas-Rachford/ADMM when $g$ is
% smooth and strongly convex, see 
% \url{https://stanford.edu/~boyd/papers/pdf/diag_scaling_DR_ADMM.pdf} as well
% as \url{https://arxiv.org/pdf/1812.00146.pdf} for more results}

%\aurelien{TODO: the theorem implicitly considers $e_k=0$. To cover DP-SGD, we should assume an upper bound on $\E[\|e_k\|]$ or $\E[\|e_k\|^2]$ and adapt the result accordingly}
\begin{thm}[Utility guarantees for noisy fixed-point
iterations]
\label{thm:utility_general}
Assume that $R$ is $\tau$-contractive with fixed point $u^*$. Let $P[\rho_
{k,b}=1]=q$ for some
$q\in(0,1]$. Then there exists a learning rate $\lambda_k=\lambda\in(0,1]$
such that the iterates of Algorithm~\ref{algo:pfix} satisfy:
	% If we set probability of choosing a block to $q$ for an operator $T$ of Lipschitz constant $\tau \in [0,1)$, the iterates verify
    \begin{equation*}%\label{eq:convergence_general}
    \begin{aligned}
        \expect\left(\left\|u_{\step+1}-u^*\right\|^{2} \mid \mathcal{F}_{0}\right) &\leqslant \left(1-\frac{q^2(1-\tau)}{8}\right)^{\step}D\\
        &+8\left(\frac{\sqrt{p}\sigma+ \zeta}{\sqrt{q}\left(1-\tau\right)}+\frac{p\sigma^2+ \zeta^2}{{q}^3(1-\tau)^3}\right)
    \end{aligned}
    \end{equation*}
        where $D\triangleq \|u_0 - u^*\|^2$, $p$ is the dimension of $u$, $\sigma^2>1-\tau$ is the variance of the
        added Gaussian noise, and $\E[\|e_k\|^2] \leq \zeta^2$ for some $\zeta
        \geq 0$.
\end{thm}
%\deb{The proof uses $\tau^2$ as contraction factor. Update the proof for $\tau$? Edited it everywhere now.}
\begin{rmk}
The assumption $\sigma^2>1-\tau$ is used for simplicity of presentation. More
generally, the
result holds true for $\sigma\sqrt{p}+\zeta > \sqrt{q}(1-\tau)$. In practice,
$\tau$ is always fairly close to $1$, hence this condition is not restrictive.
\end{rmk}

\begin{rmk}
    \cref{thm:utility_general} applies to DP-SGD on $\mu$-strongly
    convex and $\beta$-smooth objectives. Indeed, similar to \Cref
    {prop:dpsgd}, we can set $R(u) = u - \frac{2}
    {\beta + \mu}\nabla f(u; \cD)$ which is known to be $\frac{\beta-\mu}
    {\beta+\mu}$-contractive \cite{Ryu2015APO}. The first (non-stochastic)
    term recovers the classical $\bigo\big((\frac{\beta-\mu}
    {\beta+\mu})^k\big)$
    linear convergence rate of gradient descent. The
    second term, which captures the error due to stochasticity, is in
    $\bigo\big((p\sigma^2+ \zeta^2)/(1-\tau)^3\big)$. The $1/(1-\tau)^3$
    factor, which is not tight, is due to the particular choice of
    $\lambda$ we make in our analysis to get a closed-form rate in
    the general case.
\end{rmk}
%\aurelien{a bit surprising that we do not have a dependence on $p$ for the
%privacy noise. I think $\sigma^2$ should in fact be $p\sigma^2$ since we have
%Gaussian noise in each of the $p$ entries}\deb{Yes, it depends on how we think of $\sigma$. If we think each entry is Gaussian, we will get $p\sigma^2$ as you said. Here it is the upper bound on the expectation of the whole norm square.}

Theorem~\ref{thm:utility_general} shows that our noisy fixed-point iteration
enjoys a \emph{linear convergence rate} up to an additive error term. The
linear convergence rate depends on the contraction
factor $\tau$ and the block activation probability $q$.
The additive error term is ruled by the noise scale $\sigma\sqrt{p}+\zeta$,
where $\sigma$ is due to the Gaussian noise added to ensure DP and $\zeta$
captures
some possible additional error. 
Under a given privacy constraint, running more iterations requires to increase
$\sigma$ (due to the composition rule of DP), yielding a classical
privacy-utility trade-off ruled by the
number of iterations. We investigate this in details for private ADMM algorithms
in Section~\ref{sec:admm}.

% In particular, this theorem provides the ground for deriving privacy-utility trade-off when the inexact factors come from differentially-private mechanism. When adding Gaussian noise for ensuring a fixed privacy loss on the whole algorithm, running for more iteration force the noise to be higher at each step. There is thus an optimal number of steps for a given privacy budget, that run enough for converging while adding doing as few iterations as possible.

%\deb{Is this remark necessary here?}
%\aurelien{alternatively we could discuss this as a limitation in Section~6}\deb{that sounds better.}

% !TEX root = main.tex

\section{Private ADMM Algorithms}
\label{sec:admm}

We now use our general noisy fixed-point iteration framework introduced in
Section~\ref{sec:fix} to derive and analyze private ADMM algorithms for the
centralized, federated and fully decentralized learning settings.

% \aurelien{TODO: if time and space, we can present the general private ADMM
% algorithm before going to the special case of consensus}

\subsection{Private ADMM for Consensus}

Given a dataset $\cD= (d_1,\dots,d_n)$, we aim to solve
an ERM problem of the form given in
\eqref{eq:simple_erm}. This problem can be equivalently formulated as
a consensus problem \citep{admmbook} that fits the general form 
\eqref{eq:genericADMM} handled by ADMM:
\begin{mini}
    {x\in\mathbb{R}^{np}, z\in\mathbb{R}^p}{
    \label{eq:admm_consensus}\frac{1}
    {n}\sum_{i=1}^n f
    (x_i; d_i) + r(z)}{}{}
    \addConstraint{x- I_{n(p\times p)}z}{= 0, \quad}{}
\end{mini}
where $x=(x_1,\dots,x_n)^\top$ is composed of $n$ blocks (one for
each data item) of size $p$ and $I_{n(p\times p)}\in\R^{np\times p}$ denotes
$n$ stacked identity matrices of size $p\times p$. 
For convenience, we will
sometimes denote $f_i(\cdot)\triangleq f(\cdot;d_i)$.
% \aurelien{we need this sentence as we use $f_i$ in Alg 2.}

To privately solve problem \eqref{eq:admm_consensus}, we apply our
noisy fixed-point iteration (Algorithm~\ref{algo:pfix})
with the non-expansive operator $R_{\gamma p_{1}}R_{\gamma p_{2}}$
corresponding to ADMM (see Section~\ref{sec:admm_fixed}).
Introducing the auxiliary variable $u=(u_1,\dots,u_n)\in\mathbb{R}^{np}$
initialized to $u_0$ and exploiting the separable structure of the consensus
problem (see Appendix~\ref{app:algos} for details), we obtain the following 
(block-wise) updates:
\begin{align}
z_{k+1}&=\prox_{\gamma r}\big(\textstyle\frac{1}{n} \sum_{i=1}^n u_
{k,i}\big),\label{eq:main_z_update}\\
x_{k+1,i}&=\prox_{\gamma f_i}(2z_{k+1} - u_{k,i})\label{eq:main_x_update}\\
u_{k+1,i} &=u_{k,i}+2 \lambda\big(x_{k+1,i}-z_{k+1}+\textstyle\frac{1}{2}
\eta_{k+1,i}\big)\label{eq:main_u_update}.
\end{align}

From these updates and together with the possibility to randomly
sample the blocks in our general scheme, we can naturally obtain different
variants of ADMM for the centralized, federated and fully decentralized
learning. In the remainder of this section, we present these variants, the
corresponding trust models, and prove their privacy and utility guarantees.

\begin{rmk}[General private ADMM]
\looseness=-1 Our private ADMM algorithms for the consensus problem 
\eqref{eq:admm_consensus} are obtained as special cases of a private algorithm for
the more general problem \eqref{eq:genericADMM}. We
present this algorithm in
Appendix~\ref{app:general-private-admm}. In Appendix~\ref{app:privacy}, we
prove its privacy
guarantees via
a sensitivity analysis of the general update involving matrices $A$
and $B$, under the only hypothesis that $A$ is full rank. Then, we
instantiate these general results to obtain privacy guarantees for private
ADMM algorithms presented in this section.
% Beyond , our framework also
% yields a private ADMM algorithm for the more general problem 
% \eqref{eq:genericADMM}. We refer to Appendix~\ref{app:general-private-admm}
% for details.
\end{rmk}

% In this work, we focus on a particular form
% which reformulates the ERM problem \eqref{eq:simple_erm} as a consensus problem:
% In Machine Learning, we specialize this problem is the consensus, where several parties $\set{1, \dots, n}$ holds a distinct loss function, and the goal is to compute a global model, that minimizes the sum of these $n$ partial losses, with a regularizing term, leading to the following formulation:

% It is easy to see that \eqref{eq:admm_consensus} is equivalent to 
% \eqref{eq:simple_erm} as the linear constraints enforce $x-z_i=0$ for all $i$.
% It matches the general problems of minimizing two functions, by choosing the whole sum of functions as the first function $\sum_{i=1}^n g_i$ and the regularizer $r$ or the null function for the second one if there is no regularization. The linear constraint ensured that all local variables stays constrained to the global one thanks to the constraint $Mx - z = 0$ where $M$ is a matrix composed of $n$ Identity matrices stacked one upon each other.

% Now that we have defined the fixed-pont operator equivalent to ADMM, we can derive a private version as the noisy iterates of this operator.

% \begin{equation}
%     u_{k+1} = u^k + \lambda (T_{-\gamma A \triangleright f(c - \cdot ), -\gamma B \triangleright g}(u^k) + \eta_{k} - u^k ) 
% \end{equation}

\subsection{Centralized Private ADMM}

\looseness=-1 In the centralized setting, a trusted curator holds the dataset
$\cD$ and seeks to release a model trained on it with record-level DP
guarantees
\citep{chaudhuri2011Differentially}. Our private ADMM
algorithm for this centralized setting closely follows the updates
\eqref{eq:main_z_update}-\eqref{eq:main_u_update}. The version shown in
Algorithm~\ref{algo:conADMM} cycles over the $n$ blocks in a fixed order, but
thanks to the flexibility of our scheme we can also randomize the choice of
blocks at each iteration $k$, e.g., update a single random block or cycle over
a random perturbation of the blocks.
% \aurelien{add/remove details depending on
% which the variants for which we provide privacy guarantees below}
Note that at the end of the algorithm, we only release $z_K$, which is
sufficient for all practical purposes. Returning $x_K$
would violate differential privacy as its last update interacts with the data
through $\prox_{\gamma f_i}$ without subsequent random perturbation.
% From this general private algorithm we can derive a private version in the case of consensus. The $x$ update becomes a simple averaging, possibly followed by the proximal operator on the regularizer. $u$ and $z$ updates are separable per blocks, as described in Algorithm~\ref{algo:conADMM}
The privacy guarantees of the algorithm are as follows.

\begin{algorithm}[t]
    \SetKwComment{Comment}{$\triangleright$ }{}
    \DontPrintSemicolon
    \KwIn{initial vector $u^0$, step size $\lambda \in (0, 1]$, privacy noise variance $\sigma^2\geq0$, $\gamma >0$}
    \For{$k=0$ to $K-1$}{
        $\hat{z}_{k+1}=\frac{1}{n} \sum_{i=1}^{n}u_{k,i} $ \;
        $z_{k+1}=\prox_{\gamma r} \left(\hat{z}_{k+1}\right)$\;
        \For{$i=1$ to $n$}{ 
            $x_{k+1,i} = \prox_{\gamma f_i}(2z_{k+1} - u_{k,i})$\;
            $u_{k+1,i}=u_{k,i}+2 \lambda\big(x_{k+1,i}-z_{k+1}+\frac{1}
            {2} \eta_{k+1,i}\big)\text{ with } \eta_{k+1,i}\sim\gau{ \sigma^2
        \mathbb{I}_p}$ \;
        }
    }
    \Return{$z_{K}$}    
    \caption{Centralized private ADMM}
    \label{algo:conADMM}
\end{algorithm}

% \aurelien{from here: add privacy and possibly utility results and comment}

% If we adopt a block-wise approach, we can study the privacy per block.
% For cyclic updates, a block is only updated every $m$ times, so we can apply
% amplification by iteration to mitigate the privacy loss over the $m$ updates.

\begin{thm}[Privacy of centralized ADMM]
\label{thm:privacy_centralized_cons}
    Assume that the loss function $f(\cdot,d)$ is $L$-Lipschitz for any
    data record $d$ and consider record-level DP.
    Then Algorithm~\ref{algo:conADMM} satisfies $
    (\alpha,
    \frac{8\alpha KL^2\gamma^2}{\sigma^2n^2})$-RDP.
\end{thm}
\begin{proof}[Sketch of proof]
We bound the sensitivity of the ADMM operator by relying on the
structure of our updates, the strong convexity of proximal operators and
known bounds on the sensitivity of the $\argmin$ of strongly convex functions.
The result then follows from
composition.
\end{proof}
% \begin{thm}[Convergence of private centralized ADMM]
% 	If we set noise variance $\sigma^2=c^2\frac{\horizon\tau^2 \alpha}{2\eps m}$ and privacy budget	
% 	$\eps \leq \min\{\frac{2\alpha D^2}{mc_0^4c^2} \frac{\tau^5}{(1-\tau)^4},\frac{\alpha\tau}{c_0 m \horizon}\}$,	
% 	$\lambda \geq \frac{1}{1-\sqrt{\tau}}+\frac{c_0}{\tau}\sqrt{\frac{\eps m}{\horizon \alpha}}$, we obtain that for $m>1$ and horizon ${K} \leq \min\{\frac{c_0^2(1-\tau)^2}{\tau^{3}}{\frac{\eps m}{\alpha}},\|\bx_0 - \bar{\bx}\|^2(1-c_0)(1-\sqrt{\tau})\sqrt{\frac{2\eps m}{c^2\tau \alpha}}\}$
% 	\begin{equation}\label{eq:convergence_central}
% 	\expect\left(\left\|\bx_{\horizon+1}-\overline{\bx}\right\|^{2} \mid \mathcal{F}_{0}\right) = \bigO{ \exp(-2c_0\horizon)  \left(\|\bx_0 - \bar{\bx}\|^2+1+\frac{1}{(1-\sqrt{\tau})\horizon}+\frac{1}{\sqrt{\tau}\horizon^2}\right)},
% 	\end{equation}
% 	where $c_0$ and $c$ are in $(0,1]$.% and $D=\max_{\bx_0} \|\bx_0 - \bar{\bx}\|^2$.
% \end{thm}
% \deb{all the convergence results hereafter should be updated with specific $q$ and $\sigma$.}

Theorem~\ref{thm:privacy_centralized_cons} shows that the privacy loss of
centralized ADMM has a similar form as that of state-of-the-art
private gradient-based approaches like DP-SGD.
The factor $K$ comes
from the composition over the $K$ iterations, while the $L^2\gamma^2/n^2$
factor comes from the sensitivity of the ADMM operator. Crucially, the $1/n^2$
term allows for good utility when the number of data points is large
enough. We also see that, similar to output perturbation 
\citep{chaudhuri2011Differentially}, the strong convexity parameter $1/\gamma$
of the proximal updates can be used to reduce the sensitivity.

By combining Theorem~\ref{thm:privacy_centralized_cons} and our generic
utility analysis (Theorem~\ref{thm:utility_general} with $q=1$), we obtain the
following privacy-utility trade-off.

\begin{corollary}[Privacy-utility trade-off of centralized ADMM]\label{lem:priv_util_central}
	Under the assumptions and notations of Theorem~\ref{thm:utility_general}
    and~\ref{thm:privacy_centralized_cons}, setting $K$ appropriately,
    % and for number of iterations $K= \bigo\left( \log \left(\frac{L \gamma}{\left(1-\tau\right)n D} \left(\frac{p\alpha}{\varepsilon}\right)^{1/2}+\frac{L^3 \gamma^3}{\left(1-\tau\right)^3n^3 D} \left(\frac{p\alpha}{\varepsilon}\right)^{3/2}\right)\right)$,
    Algorithm~\ref{algo:conADMM} achieves\footnote{$\widetilde{\bigo}$ ignores all the logarithmic terms.}
	\begin{equation*}
		\expect\left(\left\|u_{K}-u^*\right\|^{2}\right)  = 
        \widetilde{\bigo}\bigg(\frac{\sqrt{p\alpha}L \gamma}{
        \sqrt{\varepsilon}n\left(1-\tau\right)} +\frac{p\alpha L^2 \gamma^2}
        {\varepsilon n^2 \left(1-\tau\right)^3}\bigg).
	\end{equation*}
\end{corollary}
%\aurelien{@Deb: it would be great to add a corollary giving the privacy/utility trade-off, i.e., the utility achieved for $(\alpha,\eps)$-RDP when setting $K$ to balance both terms in the utility result. A big O result would be sufficient. Otherwise, we can just say that the utility is obtained by Theorem~2 with $q=1$.}

\subsection{Federated Private ADMM}

We now switch to the Federated Learning (FL) setting 
\citep{kairouz_advances_2019}. We consider a set of $n$ users, with each user
$i$ having a local dataset
$d_i$ (which may consist of multiple data points). The function $f_i(\cdot)=f
(\cdot;d_i)$ thus represents the local objective of user $i$ on its local
dataset $d_i$. As before, we denote the joint dataset by $\cD=
(d_1,\dots,d_n)$, but we now consider user-level DP.

As commonly done in FL, we assume that the algorithm is orchestrated
by a
(potentially untrusted) central server. FL algorithms typically proceed in
rounds. At each round, each user computes in parallel a local update to the
global model based on its local dataset, and these updates are aggregated by
the server to yield a new global model.
Our federated private ADMM algorithm follows this procedure by essentially
mimicking the updates of its centralized counterpart.
Indeed, these updates can be executed in a federated fashion since (i) the
blocks $x_i$ and $u_i$ associated to each user $i$ can be updated and
perturbed locally and in parallel,
and (ii) if each user
$i$ shares $u_{k+1,i}-u_{k,i}$ with the server, then the latter can execute
the
rest of the updates to compute $z_{k+1}$. In particular, we do not need to
send $x_i$ to the server during training (the consensus is achieved
through $z$). On top of this vanilla version, we can
natively accommodate
\emph{user sampling} (often called ``client sampling'' in the literature),
which is a key property for cross-device FL as it allows to improve efficiency
and to model partial user availability \citep{kairouz_advances_2019}.
User sampling is readily obtained from our general scheme by choosing a
subset of $m$ blocks (users) uniformly at random. Algorithm~\ref{algo:fedADMM}
gives the complete procedure.

% Users compute local updates based on their local dataset and
% send these updates to the server, which aggregates them and send back updated
% information.

\looseness=-1 The privacy guarantees of FL algorithms can be analyzed at two
levels
\citep{Noble2022a}. The first level, corresponding to local DP \cite{ldp1,ldp2}, is the privacy of each user with respect to the
server (who observes the sequence of invidivual updates) or anyone
eavesdropping on the communications.
The second level, corresponding to central DP, is the privacy guarantee of
users with respect to a third party observing only the final model.
Our algorithm naturally provides these two levels of privacy,
% users have local DP guarantees that protect them against an honest-but-curious
% server , and (stronger) central
% DP guarantees that protect them against attacks on the final released model.
as shown in the following theorem.

% Fully decentralized and block-wise update illustrates that our approach to ADMM allows to easily reuse privacy amplification by iteration. However, a lot of applications relies on federated learning with a trusted server, where every node or a sampling of the nodes participates at each round. In this case, having two levels of privacy guarantees allows taking into account different kind of vulnerabilities. Personal data can be attacked in case of interception of the messages containing the updates send to the server or if the server itself is malicious. This risk can only be mitigated with local differential privacy, but in case where the level of trust is high, a high privacy budget can be tolerated.

\begin{algorithm}[t]
    \SetKwComment{Comment}{$\triangleright$ }{}
    \DontPrintSemicolon
    \KwIn{initial point $z_0$, step size $\lambda \in (0, 1]$, privacy noise
    variance $\sigma^2\geq0$, parameter $\gamma >0$, number of sampled users
    $1\leq m \leq n$}
    \textbf{Server loop:}\;
    \For{$k=0$ to $K-1$}{
    %Sample randomly a subset $U$ of users\;
    Subsample a set $S$ of $m$ users\;
    \For{$i \in S$}{
        $\Delta u_{k+1,i} = ~$\textbf{LocalADMMstep}$(z_k,i)$\;
    }
    $\hat{z}_{k+1} = z_k + \frac{1}{n}\sum_{i \in S}\Delta u_{k+1,i} $\;
    $z_{k+1} = \prox_{\gamma r}(\hat{z}_{k+1})$\;
    }
    \Return{$z_K$\;}
    \vspace*{.2cm}
    \textbf{LocalADMMstep}$(z_k,i)$\textbf{:}\;
    Sample $\eta_{k+1,i}\sim\gau{ \sigma^2
        \mathbb{I}_p}$\;
    $x_{k+1,i} = \prox_{\gamma f_i}(2z_{k} - u_{k,i})$\;
    $u_{k+1,i}=u_{k,i}+2 \lambda\left(x_{k+1,i}-z_{k}+\frac{1}{2} \eta_
    {k+1,i}\right)$\;
    \Return{$u_{k+1,i}-u_{k,i}$}
    \caption{Federated private ADMM}
    \label{algo:fedADMM}
\end{algorithm}

% Another threat comes from the release and the use of the model itself: contrary to the previous threat, we are sure that the model is used and shared, so every malicious user can access it, and thus it makes sense to require stronger guarantees at this central level. 

% Our algorithm naturally provides these two level amplifications for federated learning.

\begin{thm}
    \label{thm:privacy_fed}
    Assume that the loss function $f(\cdot,d)$ is $L$-Lipschitz for any
    local dataset $d$ and consider user-level DP. Let $K_i$ be the number of
    participations of user $i$.
    Then, Algorithm~\ref{algo:fedADMM} satisfies
    $(\alpha,\frac{8\alpha K_iL^2\gamma^2}{\sigma^2})$-RDP for user $i$ in the
    local model. Furthermore, if  $m<n/5$ and $\alpha \leq \left(M^2
    \sigma^2
    / 2-\log \left(5 \sigma^2\right)\right) /\left(M+\log (m \alpha/n)+1
    /\left
    (2 \sigma^2\right)\right)$ where $M=\log (1+1 /(\frac{m}{n}(\alpha-1)))$,
    then it
    also satisfies $(\alpha,\frac{16\alpha KL^2\gamma^2}{\sigma^2n^2})$-RDP
    in the central model.
\end{thm}
\begin{proof}[Sketch of proof]
The local privacy guarantee follows from a sensitivity analysis,
similarly to the centralized case. Then, we obtain the central guarantee by
using amplification by subsampling and the aggregation of user
contributions.
\end{proof}

As expected, the local privacy guarantee does
not
amplify with the number of users $n$: since the server observes all individual
updates, privacy only relies on the noise added locally by the user. In
contrast, the central privacy guarantee benefits from both amplification by
subsampling \cite{amp_sub_mironov} thanks to user sampling (which gives a
factor $m^2/n^2$) and by
aggregation of the contributions of the $m$ sampled users (which gives
a factor $1/m^2$). In the end, we thus recover the privacy guarantee of
the centralized algorithm with the $1/n^2$ factor.
We stress that the restriction on $m/n$ and $\alpha$ in
Theorem~\ref{thm:privacy_fed} is only to obtain the simple closed-form
solution, as done in other works \citep[see e.g.][]{Altschuler2022}. In
practice, privacy accounting is done numerically, see
Appendix~\ref{app:privacy_fl} for details.

\begin{rmk}[Secure aggregation]
    Our federated ADMM algorithm is compatible with the use of secure
    aggregation \citep{Bonawitz2017a}. This allows the server to obtain
    $\sum_{i \in S}\Delta u_{k+1,i}$ without observing
    individual user contributions. In this case, the sensitivity is divided by
    $m$ and the privacy of users with respect to the server is thus amplified
    by a factor $1/m^2$. Therefore, for full participation ($m=n$), we
    recover the privacy guarantee of the centralized case.
\end{rmk}

We provide the privacy-utility trade-off by resorting to
Theorem~\ref{thm:utility_general}, where we fix $q=m/n=r$ with $r\in(0,1/5]$.

%\aurelien{which case is this? probably central DP? why is it much worse than
%centralized case? (in the dependence on $n$ in particular?)}
\begin{corollary}[Privacy-utility trade-off of federated ADMM in the central
model]\label{lem:priv_util_fed}
	Under the assumptions and notations of Theorem~\ref{thm:utility_general} and~\ref{thm:privacy_fed},
    % and for number of iterations $K= \bigo\left( \log \left(\frac{L \gamma}{\left(1-\tau\right)n D} \left(\frac{p\alpha}{\varepsilon}\right)^{1/2}+\frac{L^3 \gamma^3}{\left(1-\tau\right)^3n^3 D} \left(\frac{p\alpha}{\varepsilon}\right)^{3/2}\right)\right)$,
    setting $K$ appropriately, and also $m= r n$ for $r \in (0,1/5)$, Algorithm~\ref{algo:fedADMM} achieves
    \begin{align*}
    \expect\left\|u_{K}-u^*\right\|^{2}  = \widetilde{\bigo}\bigg(\frac{
    \sqrt{p\alpha}L \gamma}{\sqrt{\varepsilon r}n \left(1-\tau\right)} +
    \frac{p \alpha L^2 \gamma^2}{\varepsilon r^2 n^2 \left
    (1-\tau\right)^3}\bigg).
    \end{align*}
\end{corollary}
%\aurelien{@Deb: add corollary giving the privacy/utility trade-off?}

\subsection{Fully Decentralized Private ADMM}

% Federated learning considers the setting where a set of users with their own
% data seek to collaboratively learn a model while keeping their data
% decentralized \citep{kairouz_advances_2019}. Here, we consider a set of $n$
% users, and each user $i$ has a local dataset $d_i$ (which may consist of
% multiple data points). As before, we denote the joint dataset by $\cD=
% (d_1,\dots,d_n)$.

% Fully decentralized learning algorithms

Finally, we consider the fully decentralized setting. The setup is similar to
the one of federated learning investigated in the previous section, except
that there is no central server. Instead, users communicate in a peer-to-peer
fashion
along the edges of a network graph. Fully decentralized algorithms are popular
in machine learning due to their good scalability 
\cite{Lian2017CanDA,koloskova2021unified},
and were recently shown to provide privacy amplification 
\citep{Cyffers2020PrivacyAB,Cyffers2022b}.

We consider here the complete network graph (all users can communicate with
each
others). Instantiating our general private ADMM algorithm with
uniform subsampling of a single block at each iteration, we directly obtain a
fully
decentralized version of ADMM (Algorithm~\ref{algo:p2ppADMM}). The algorithm
proceeds as follows. The model $z_0$ is initialized at some user $i$. Then, at each
iteration $k$, the user with the model $z_k$ performs a local noisy
update using its local dataset $d_i$, and then sends the resulting $z_{k+1}$
to a randomly chosen user. In other words, the model is updated by
following a random walk.
This random walk paradigm is quite popular in decentralized algorithms
\citep{doi:10.1137/08073038X,walkmanMao2020,doi:10.1137/08073038X}. In
particular, it requires little computation and communication compared to other
algorithms with more redundancy (such as gossip). Alleviating the need of
synchronicity and full availability for the users can lead to faster algorithms in
practice.

% The previous result illustrates how distinguishing blocks is easy to do in our setting and leads to an optimal factor amplification of $m$ for $m$ different blocks. As only the $x$ value need to be passed from a step to the next, this algorithm can be fully decentralized. Each node is associated with a given block, and a token holding current $x$ value follows a random walk on the complete graph. At each step, the node receiving the token update privately is $z$ and $\phi$ value to update $x$ value and forwards it. Indeed, the $x$ value can be updated locally by noting that he sum is only modified by the difference between the previous and the new values of the current node.

% This random walk paradigm is quite popular in decentralized algorithms as it requires little computation and communication compared to more redundant algorithm such as gossip. In practice, alleviating the need of synchronicity is a significant gain that can lead to faster algorithms as well \cite{incrementalADMM, doi:10.1137/08073038X}.

It is easy to see that our fully decentralized algorithm enjoys the same local
privacy guarantees as its federated counterpart (see
Theorem~\ref{thm:privacy_fed}). This provides a baseline protection against
other users, and more generally against any adversary that would eavesdrop
on all
messages sent by the users. Yet, this guarantee can be quite pessimistic if
the goal is to protect against other users in the system. Indeed, it is
reasonable to
assume that each user $i$ has only a limited view and only
observes the messages it receives, without knowing the random path
taken by
the model between two visits to $i$. To capture this and improve privacy
guarantees compared to the local model, we rely on the notion of 
\emph{network DP}, a relaxation of local DP recently introduced by 
\citet{Cyffers2020PrivacyAB}.
% does not know the 
% In the case of fully decentralized setting, the threat of privacy come from the other nodes that might try to extract as much information as possible from the token. So, we need to enforce privacy guarantee towards the knowledge acquired by a node $u$ (attacker) on a node $v$ during the algorithm. We rely on a variant of RDP introduced in \citep{Cyffers2020PrivacyAB} adapted to this setting.

\begin{algorithm}[t]
	\SetKwComment{Comment}{$\triangleright$ }{}
	\DontPrintSemicolon
	
	\KwIn{initial points $u_0$ and $z_0$, step size $\lambda \in (0, 1]$,
		privacy noise variance $\sigma^2\geq0$, $\gamma >0$}
	\For{$k=0$ to $K-1$}{
		Let $i$ be the currently selected user\;
        Sample $\eta_{k+1,i}\sim\gau{ \sigma^2
        \mathbb{I}_p}$\;
		$x_{k+1,i} = \prox_{\gamma f_i}(2z_{k} - u_{k,i}) $\;
		$u_{k+1,i}=u_{k,i}+2 \lambda\left(x_{k+1,i}-z_{k}+\frac{1}
		{2} \eta_{k+1,i}\right)$\;
		$\hat{z}_{k+1} = z_k + \frac{1}{n} (u_{k+1,i}-u_{k,i})$\;
		% \Comment*[f]{equal to global average} \;
		$ z_{k+1}=\prox_{\gamma r} \left(\hat{z}_{k+1}\right)$\;
		Send $z_{k+1}$ to a random user\;
	}
	
	\caption{Fully decentralized private ADMM}
	\label{algo:p2ppADMM}
\end{algorithm}

\begin{defn}[Network Differential Privacy]
    \label{def:network_dp}
    An algorithm $A$ satisfies $(\alpha, \eps)$-network RDP if for all pairs of
    distinct users $i, j\in \{1,\dots,n\}$ and all pairs of neighboring
    datasets $\cD \sim_i \cD'$ differing only in the
    dataset of user $i$, we have:
    \begin{equation}
    \label{eq:network-dp}
    D_{\alpha}\left(\Obs{j}(A(\cD)) \| \Obs{j}(A\left(\cD')\right)\right)
\leq \eps.
    \end{equation}
    where $\Obs{j}$ is the view of user $j$.
\end{defn}
In our case, the view $\Obs{j}$ of user $j$ is limited to $\Obs{j}(A
(\cD)) = (z_{k_l(j)})_{l=1}^{K_j}$
where $k_l(j)$ is the time of $l$-th contribution of user $j$ to the
computation, and
$K_j$ is the total number of times that $j$ contributed during the execution
of
algorithm. We can show the following network DP guarantees.

\begin{thm}
    \label{thm:privacy_dec}
    Assume that the loss function $f(\cdot,d)$ is $L$-Lipschitz for any
    local dataset $d$ and consider user-level DP.
    Let $\alpha >1, \sigma > 2L\gamma \sqrt{\alpha (\alpha -1)}$ and
    $K_i$ the
    maximum number of contribution of a user. Then
    Algorithm~\ref{algo:p2ppADMM} satisfies $(\alpha, \frac{8 \alpha K_i L^2
    \gamma^2\ln n} {\sigma^2 n})$-network RDP.
\end{thm}
\begin{proof}[Sketch of proof]
Fixing a single participation of a given user (say $i$), we have the same
local
privacy
loss as in the federated case. We then control how much this leakage
decreases when the information reaches another user (say $j$). To do this, we
first quantify the leakage when the $z$ variable is seen by user $j$ after $m$
steps by relying on privacy amplification by
iteration. Then, thanks to the randomness of the path and the weak convexity
of the Rényi divergence, we can average the different possible lengths $m$ of
the path between users $i$ and $j$ in the complete graph. We conclude by
composition over the number $K_i$ of participations of a user.
\end{proof}

\looseness=-1 Remarkably, Theorem~\ref{thm:privacy_dec} shows that thanks to
decentralization, we obtain a privacy amplification of $O(\ln n / n^2)$
compared to the local DP guarantee. This amplification factor is of the same
order as the one proved by \citet{Cyffers2020PrivacyAB} for a random walk
version of DP-SGD, and matches the privacy guarantees of the
centralized case up to a $\ln n$ factor. To the best of our knowledge, this is
the first result of this kind for fully decentralized ADMM, and the first
application of privacy amplification by iteration to ADMM.

As before, we obtain the privacy-utility trade-off by resorting to
Theorem~\ref{thm:utility_general}, but this time with $q=1/n$.

%\aurelien{this seems incorrect: I think it misses the fact that each user will contribute only roughly $K/n$ times (and not $K$ times))}
\begin{corollary}[Privacy-utility trade-off of decentralized ADMM]
\label{lem:priv_util_dec}
	Under the assumptions and notations of Theorem~\ref{thm:utility_general}
    and~\ref{thm:privacy_dec},
    % and for number of iterations $K=\bigo\left(\log\left(\frac{L \gamma}{\left(1-\tau\right)D} \left(\frac{p\alpha \ln n}{\varepsilon}\right)^{1/2}+\frac{L^3 \gamma^3}{\left(1-\tau\right)^3 D} \left(\frac{p\alpha \ln n}{\varepsilon}\right)^{3/2}n^2\right)\right)$
    setting $K$ appropriately
    Algorithm~\ref{algo:p2ppADMM} achieves
	\begin{align*}
		\expect\left(\left\|u_{K}-u^*\right\|^{2}\right)  = 
        \widetilde{\bigo}\bigg(\frac{\sqrt{p \alpha}L \gamma}{
        \sqrt{\varepsilon n}\left(1-\tau\right)} +\frac{p\alpha  L^2 \gamma^2}
        {\varepsilon n \left(1-\tau\right)^3} \bigg).
	\end{align*}
\end{corollary}

\begin{rmk}[Utility guarantees for centralized, federated, and decentralized
settings.]
From Corollary~\ref{lem:priv_util_central}, we observe that the utility
for the centralized setting is $\widetilde{\bigo}\left( \sqrt{\frac{p \alpha}
{\varepsilon}}\frac{1}{n}\right)$ (in the regime $n \gg p$). On the other
hand, the utility for the decentralized setting is $
\widetilde{\bigo}\left(\sqrt{\frac{p \alpha}{\varepsilon n}}\right)$. This
difference captures the shift in hardness from the
centralized setting to the decentralized one. For the federated learning
setting, the utility is $\widetilde{\bigo}\left( \sqrt{\frac{p \alpha}
{\varepsilon m n}}\right)$, where $m$ is the number of sampled users at each
step. Thus, if $m=n$ (all users contribute at each step), we recover
the utility of the centralized setting. Instead, if $m=1$, we are back to
the utility of the decentralized setting. These observations demonstrate that
our results on the privacy-utility trade-offs reasonably quantify the relative
hardness of these three settings.
\end{rmk}

\begin{rmk}[Clipping updates]
In practical implementations of private optimization algorithms, it is very
common to use a form of clipping to enforce a tighter sensitivity bound than
what can be guaranteed theoretically \citep[see e.g.,][]
{abadi2016Deep,dp_fedavg_user_level}. In our private ADMM algorithms, this can
be done by clipping the quantity $(x_{k+1,i}-z_{k+1})$ in the $u$-update, see
\Cref{app:expes} for details.
\end{rmk}

% \begin{corol}[Convergence of DP-ADMM]
% 	If the  total number of iterations $T \geq \frac{\eps}{8m\tau\alpha}$, noise variance $\sigma^2 = \frac{T\tau \alpha}{2\eps} $ , then for any $u \in (0,1]$, we can choose a privacy budget $\eps \leq u^2 \frac{\tau\alpha T}{32(1-\tau)}$ such that
% 	\begin{equation}\label{eq:linear_convergence_dp}
% 	\expect\left(\left\|\bx_{n+1}-\overline{\bx}\right\|^{2} \mid \mathcal{F}_{0}\right) = \bigO{ \exp{(- c_1 n)}	\left(\left\|\bx_{0}-\overline{\bx}\right\|^{2} + c_2\right)},
% 	\end{equation}
% 	where\todo{edit Thm 3 that further by putting $p=1/m$.}
% 	\begin{align*}
% 	c_1 &\leq \frac{1}{2}+u,\\
% 	c_2 &\triangleq \tilde{\bigo}\left(\frac{4\exp(1+2u)}{m}\left(1+\frac{1}{u\sqrt{m-1}}\right)\right).
% 	\end{align*}
% \end{corol}

%\aurelien{@Deb: add corollary giving the privacy/utility trade-off?}

\subsection{Numerical Illustration}

We refer to \Cref{app:expes} for a numerical illustration of our private ADMM
algorithms on a simple Lasso problem. We empirically observe that private ADMM
tends to outperform DP-SGD in high-privacy regimes.

\section{Conclusion}

In this work, we provide a unifying view of private optimization algorithms
by framing them as noisy fixed-point iterations. The advantages of this novel
perspective for privacy-preserving machine learning are at least two-fold.
First, we give utility guarantees based
only on very general assumptions on the underlying fixed-point operator,
allowing us to cover many algorithms. Second, we show that we can derive new private algorithms by instantiating our general scheme with
particular fixed-point operators. We illustrate this through the design of
novel private ADMM algorithms for the centralized, federated and fully
decentralized learning and the rather direct analysis of their privacy and
utility guarantees.
We note that the generality of our approach
may sometimes come at the cost of the tightness of utility guarantees, as we
do not exploit the properties of specific algorithms beyond their contractive
nature.

\looseness=-1 We believe that our framework provides a general and principled
approach to
design and analyze novel private optimization algorithms by leveraging the
rich literature on fixed-point iterations \citep{fixedpointDS}.
In future work, we would like to further broaden the applicability of our
framework by proving (weaker) utility guarantees for $\lambda$-averaged
operators that are non-expansive but not contractive. To achieve this, a
possible direction is to extend the sublinear rates of \citep{cvnonexpansivePeyre}
to block-wise iterations.

% \begin{rmk}[On the tightness of utility guarantees]
% 	Theorem~\ref{thm:utility_general} relies only on very general assumptions on the
% 	underlying fixed-point operator, allowing it to cover many algorithms.
% 	This generality may come at the cost of tightness. It might be possible to
% 	obtain better bounds by exploiting the properties of specific
% 	algorithms.
% \end{rmk}

% \aurelien{discuss the perspective of proving (sublinear) convergence rates for
% the non-expansive case, which would require extending some existing analysis 
% (we had found two papers)}

\section*{Acknowledgments}

This work was supported by grant ANR-20-CE23-0015 (Project PRIDE), the
ANR-20-THIA-0014 program ``AI\_PhD@Lille'', the ANR 22-PECY-0002 IPOP 
(Interdisciplinary Project on Privacy) project of the Cybersecurity PEPR, and the ANR JCJC ANR-22-CE23-0003-01 for the REPUBLIC project.

% !TEX root = main.tex

% % In the unusual situation where you want a paper to appear in the
% % references without citing it in the main text, use \nocite
% \nocite{langley00}
% \clearpage
\bibliography{biblio}
\bibliographystyle{icml2023}

%%%%%%%%%%%%%%%%%%%%%%%%%%%%%%%%%%%%%%%%%%%%%%%%%%%%%%%%%%%%%%%%%%%%%%%%%%%%%%%
%%%%%%%%%%%%%%%%%%%%%%%%%%%%%%%%%%%%%%%%%%%%%%%%%%%%%%%%%%%%%%%%%%%%%%%%%%%%%%%
% APPENDIX
%%%%%%%%%%%%%%%%%%%%%%%%%%%%%%%%%%%%%%%%%%%%%%%%%%%%%%%%%%%%%%%%%%%%%%%%%%%%%%%
%%%%%%%%%%%%%%%%%%%%%%%%%%%%%%%%%%%%%%%%%%%%%%%%%%%%%%%%%%%%%%%%%%%%%%%%%%%%%%%
\newpage
\appendix
\onecolumn

% \input{ampli}

% !TEX root = main.tex

\section{Generic Utility Analysis of Private Fixed Point Iterations (Algorithm~\ref{algo:pfix})}
\label{app:convergence}

\subsection{Existing Result of~\citet{combettes2019stochastic}} 
Our convergence analysis leverages the generic convergence result
of~\citet{combettes2019stochastic} for stochastic quasi-Fej\'er
type block-coordinate fixed-point operators. Here, we briefly summarize their
result (Theorem 3.1 in \citealp{combettes2019stochastic}) before deriving our specific analysis.

\begin{thm}[Mean-square convergence of stochastic quasi-Fej\'er type
block-coordinate iterations, \citealp{combettes2019stochastic}]\label{thm:generic_combettes}
	The update rule of the stochastic quasi-Fej\'er type block-coordinate iterations is given by
	\begin{equation}\label{eq:quasifejer}
	\bx_{\step+1,\block}=\bx_{\step,\block}+\rho_{\step,\block} \lambda_
	{\step}\left(\operator_{\step,\block}\left(\bx_{\step}\right)+e_{i,
	\step}-\bx_{i, \step}\right).
	\end{equation}
	Here, $\block \in [\blocks]$ denotes the $\block$-th coordinate (block) of
	$\bx\in\cU = \cU_1 \times \dots \times \cU_B$, i.e. $\bx_\step =[\bx_
	{\step,1}, \ldots, \bx_{\step,b}, \ldots, \bx_{\step,\blocks}]$, and
	$\step$ denotes the number of iterations. We assume that the operators
	$(\operator_\step)_{\step\in\mathbb{N}}$ are quasi-non-expansive with
	common fixed point $u^*$ such that:
	\begin{equation}\label{eq:coordinate_lipschitz}
	\left\|\operator_{\step} (\bx)-\bx^*\right\|^{2} \leq \sum_{\block=1}^
	{\blocks} \tau_
	{\step,\block}\left\|\bx_{\block}-\bx^*_{\block}\right\|^
	{2},~~\forall \step \in \N, \forall \bx \in \cU, \text{ and }
	\exists~\tau_{\step,\block} \in [0,1).
	\end{equation}
	%With the norm
	%\[
	%(\forall x \in \mathbf{H}) |||x|||^2 =\sum_{i=1}^{m} \omega_{i} \norm{x_i}^2
	%\]
	Let $(\mathcal{F}_{\step})_{\step\in\mathbb{N}}$ be a sequence of
	sub-sigma-algebras of $\mathcal{F}$ such that $\forall \step\in
	\mathbb{N}: \sigma(\bx_0,\dots,\bx_\step)\subset\mathcal{F}_{\step}\subset
	\mathcal{F}_{\step+1}$.

	Given this structure, we assume that the following conditions hold:
	
	[a] $\inf _{\step \in \N} \lambda_{\step}>0$.
	
	[b] There exists a sequence of non-negative real numbers $\left(\alpha_
	{\step}\right)_{\step \in \N}$ such that $\sum_{\step \in \N} \sqrt{\alpha_{\step}}<+\infty$, and $\expect\left(\left\|e_{\step}\right\|^{2} \mid \mathcal{F}_{\step}\right) \leq \alpha_{\step}$ for every $\step \in \N$.
	
	[c] For every $\step \in \N, \mathcal{E}_{\step}=\sigma\left(\rho_{\step}\right)$ and $\mathcal{F}_{\step}$ are independent.
	
	[d] For every $\block \in\{1, \ldots, \blocks\}, \mathrm{p}_
	{\block}=\Pro\left[\rho_{0,\block}=1\right]>0$.
	
	Under the assumptions [a]-[d], the iteration defined in Equation~\eqref{eq:coordinate_lipschitz} satisfies almost surely
	\begin{equation}\label{eq:generic_mean_sq_convergence}
	\sum_{\block=1}^\blocks\omega_\block \expect\left(\left\|\bx_
	{\step+1,\block}-\bx^*_{\block}\right\|^{2} \mid \mathcal{F}_
	{0}\right) \leq \left(\prod_{j=0}^{\step} \chi_{j}\right)	\left(\sum_
	{\block=1}^\blocks\omega_\block\left\|\bx_{0,\block}-\bx^*_
	{\block}\right\|^{2}\right)+\bar{\eta}_{\step},\quad \forall \step \in \N.
	\end{equation}
	Here,
	\begin{align*}
	\chi_{\step} &=1-\lambda_{\step}\left(1-\mu_{\step}\right)+\sqrt{\xi_{\step}} \lambda_{\step}\left(1-\lambda_{\step}+\lambda_{\step} \sqrt{\mu_{\step}}\right) \\
	\bar{\eta}_{\step} &=\sum_{j=0}^{\step}\left[\prod_{\ell=j+1}^{\step} \chi_{\ell}\right] \lambda_{j}\left(1-\lambda_{j}+\lambda_{j} \sqrt{\mu_{j}}+\lambda_{j} \sqrt{\xi_{j}}\right) \sqrt{\xi_{j}}\\
	\xi_{\step} &= \alpha_{\step}\max _{1 \leq \block \leq \blocks}
	\omega_\block \\
	\mu_{\step} &=1-\min _{1 \leq \block \leq \blocks} \left(p_{\block}-
	\frac{\tau_{\step,\block}}{\omega_\block}\right)\\
	\max _{1 \leq \block \leq \blocks} &\varlimsup \tau_{\step,\block} <
	\omega_\block p_\block
	\end{align*}
\end{thm}

In this paper, we leverage this result to derive our generic convergence
analysis for the private fixed-point iteration (Algorithm~\ref{algo:pfix}),
which we then instantiate to the three types of private ADMM algorithms we
introduce (Section~\ref{sec:admm}).

\subsection{Proof of Theorem~\ref{thm:utility_general}}
For ease of calculations, we mildly restrict the coordinate-wise contraction
assumption made in Theorem~\ref{thm:generic_combettes} by the
following assumption of global contraction.
\begin{assum}[Global contraction constant]\label{ass:glob_lip}
	In our analysis, we assume that there exists a global contraction constant
	$\tau \in [0,1)$ for the contraction operator $\operator_k$.
	Mathematically,
	\begin{equation}%\label{eq:coordinate_lipschitz}
	\left\|\operator_{\step} (u)- u^*\right\|^{2} \leq \sum_{\block=1}^
	{\blocks} \tau_{\step,\block}^2 \left\|u_{\block}-u^*_{\block}\right\|^{2}
	\leq \tau^2 \left\|u-u^*\right\|^{2},~~\forall \step \in \N, \forall \bx
	\in
	\cU.
	\end{equation}
\end{assum}%\todo{change $\tau$ to $\tau$}

%Also, as the noise introduced in all the settings are zero-centered, we obtain $\expect\left(\left\|a_{\step}\right\|^{2} \mid \mathcal{F}_{\step}\right) \leq \sigma^2$, where $\sigma^2$ is the noise variance decided by the privacy-preserving mechanism. Thus, $\alpha_{\step} \leq \sigma^2$ for all $\step \in \N$.

%\aurelien{TODO: restate the general theorem of Combettes/Pesquet; update error term to accommodate a bound on $\mathbb{E}[\|e_k\|]$; proofread}

%\aurelien{if we want to restate our Theorem we should use a properenvironment which re-uses the same theorem number}
\begin{reptheorem}{thm:utility_general}
	Assume that $R$ is a $\tau$-contractive operator with fixed point $u^*$
	for $\tau \in [0,1)$. Let $P[\rho_{k,b}=1]=q$ for some $q\in(0,1]$. Then
	there exists a learning rate $\lambda_k=\lambda\in(0,1]$ such that the
	iterates of Algorithm~\ref{algo:pfix} satisfy:
	\begin{equation}\label{eq:convergence_general}
	\expect\left(\left\|u_{\step+1}-u^*\right\|^{2} \mid \mathcal{F}_{0}\right) \leqslant \left(1-\frac{q^2(1-\tau)}{8}\right)^{\step}D
	+8\left(\frac{\sqrt{p}\sigma+ \zeta}{\sqrt{q}\left(1-\tau\right)}+\frac{p\sigma^2+ \zeta^2}{{q}^3(1-\tau)^3}\right)
	%\left(1 - \frac{q^2(1-\tau)}{8}\right)^{\step}\left(D+\frac{2(\sigma\sqrt{p} + \zeta)}{\sqrt{q}\left(1-\tau\right)}\right)	+\frac{8(p\sigma^2+ \zeta^2)}{{q}^3(1-\tau)^3}+\frac{4}{\left(1-\tau\right)}
	\end{equation}
	where $D=\max_{u_0} \|u_0 - u^*\|^2$ is the diameter of the domain, $p$ is the dimension of $u$, $\sigma^2 > 1 - \tau$ is the variance of Gaussian noise, and $\E[\|e_k\|^2] \leq
	\zeta^2$ for some $\zeta \geq 0$.
\end{reptheorem}
%\textbf{Remark:} Hereafter, we use $\tau^2$ instead of $\tau$. We assume all the aforementioned variables are invariant for all the iterations and co-ordinates.

\begin{proof}
We observe that Algorithm~\ref{algo:pfix} satisfies the assumptions of
Theorem~\ref{thm:generic_combettes} if we specify $p_{\block}=q$, $\omega_
{\block}=\frac{1}{q}$, and $\mu=1-q\left(1-\tau\right)$. Since $\xi \geq \frac{1}{q}\E[\|e_k+\eta_k\|^2]$, $\E[\|\eta_k\|^2] \leq p \sigma^2$ as zero-mean Gaussian noise are added independently to each dimension, and $\E[\|e_k\|^2] \leq \zeta^2$, we can assign $\xi = \frac{p\sigma^{2}
+ \zeta^2}{q}$. For ease of calculations, hereafter, %we use an upper bound of $\xi$, which is $ \frac{(\sqrt{p}\sigma + \zeta)^2}{{q}}$, and 
we refer to $\xi$ as $\sigma_{1}^2$.

%For the generic analysis, we set $p_{i}=q$, $\omega_{i}=\frac{1}{q}$, $\mu=1-q\left(1-\tau\right)$, and $\xi=\frac{\sigma^{2}}{q}$ for some $q >0$.
\paragraph{Step 1: Instantiating the mean-square convergence result.}
By substituting the aforementioned parameters in Equation~\eqref{eq:generic_mean_sq_convergence}, we obtain
\begin{align}
	\expect\left(\left\|u_{\step+1}-u^*\right\|^{2} \mid \mathcal{F}_{0}\right) &\leq \chi^\step	\left\|u_{0}-u^*\right\|^{2}+q {\eta} \notag\\
	&\leq \chi^\step D + q \eta.\label{eq:start_generic_proof}
\end{align}
Here, \begin{align*}
\chi &=1-\lambda q\left(1-\tau\right)+\lambda \sigma_1\left(1-\lambda+\lambda \sqrt{1-q\left(1-\tau\right)}\right) \\
& =1-\lambda\left(1-b^{2}\right)+\lambda\sigma_{1}(1-\lambda+\lambda b)\\
&=1+\lambda\left(\sigma_{1}-\left(1-b^{2}\right)\right)-\lambda^{2} \sigma_{1}(1-b),
\end{align*}
and
\begin{align}
\eta & =\sum_{i=0}^{\step} \chi^{\step-i-1} \lambda \sigma\left(1+\lambda\left
(\sigma_{1}-(1-b)\right)\right) \nonumber\\
& =\frac{\frac{1}{\chi}-\chi^{\step}}{1-\chi}\left(\chi-1+\lambda\left(1-b^{2}\right)+\sigma_{1}^{2} \lambda^{2}\right) \nonumber\\
& =\left(\chi^{\step}-\frac{1}{\chi}\right)\left(1-\frac{\sigma_{1} \lambda\left(\sigma_{1} \lambda+\frac{1-b^{2}}{\sigma_{1}}\right)}{\sigma_{1} \lambda\left((1-b) \lambda+\frac{1-b^{2}}{\sigma_{1}}-1\right)}\right) \nonumber\\
& =\left(\chi^{\step}-\frac{1}{\chi}\right)\left(1-\frac{\lambda \sigma_{1}+
\frac{1-b^{2}}{\sigma_{1}}}{(1-b) \lambda+\frac{1-b^{2}}{\sigma_
{1}}-1}\right).\label{eq:eta1}
\end{align}
For simplicity, we introduce the notation $b \triangleq \sqrt{1-q\left(1-\tau\right)}$. We observe that $b \in [0,1)$ as $ q \in (0, 1]$ and $\tau \in [0,1)$.

\paragraph{Step 2: Finding a `good' learning rate $\lambda$.} %By some calculations, we select $\lambda = \frac{1}{1-b}\left(1- \frac{q}{2(1+c)}\right)$.\todo{extend this using later notes}  
First, we assume that there exists a $c > 0$, such that the noise variance can be rewritten as $\sigma_{1} = (1+c)(1-\tau)$.
From Lemma~\ref{lemma:lrate}, we obtain that \[\lambda \in \left( \frac{1+c-q}{(1+c)(1-b)}, \frac{1+c-q}{(1+c)(1-b)}\left(\frac{1}{2}+  \frac{1}{2}\sqrt{1 + 4 \frac{(1+c)(1-b)}{(1-\tau)(1+c-q)^2}}\right)\right).\]

For ease of further calculations, we fix $\lambda = \frac{1}{1-b}\left(1- \frac{q}{2(1+c)}\right)$. 
Before proceeding further, we prove that this choice of $\lambda$ belongs to
the desired range.
We begin by observing that 
\[
\frac{1-b}{1-\tau}=\frac{1-\sqrt{1-q\left(1-\tau\right)}}{\left(1-\tau\right)} \geqslant \frac{q\left(1-\tau\right)}{2(1-\tau)}=\frac{q}{2}.
\]
The inequality holds due to concavity of the square root, specifically $\sqrt{1-x} \leq 1 - \frac{x}{2}$.

Thus,
\begin{align*}
	\frac{1+c-q}{(1+c)(1-b)}\left(\frac{1}{2}+  \frac{1}{2}\sqrt{1 + 4 \frac{(1+c)(1-b)}{(1-\tau)(1+c-q)^2}}\right) & \geq \frac{1+c-q}{(1+c)(1-b)}\left(\frac{1}{2}+  \frac{1}{2}\sqrt{1 +  \frac{2q(1+c)}{(1+c-q)^2}}\right)\\
	&= \frac{1+c-q}{(1+c)(1-b)}\left(\frac{1}{2}+  \frac{1}{2}\sqrt{1 +  \frac{2q(1+c-q)}{(1+c-q)^2}+ \frac{2q^2}{(1+c-q)^2}}\right)\\
	&> \frac{1+c-q}{(1+c)(1-b)}\left(\frac{1}{2}+  \frac{1}{2}\sqrt{1 +  \frac{2q}{(1+c-q)}+ \frac{q^2}{(1+c-q)^2}}\right)\\
	&= \frac{1+c-q}{(1+c)(1-b)}\left(1 +  \frac{q}{2(1+c-q)}\right)\\
	&= \frac{1}{1-b}\left(1- \frac{q}{2(1+c)}\right).
\end{align*}
% For further calculations, we consider $\lambda = \frac{1}{1-b}\left(1- \frac{q}{2(1+c)}\right)$. 

\paragraph{Step 3: Understanding the impact of the noise term $\eta$.}
First, we investigate the term $\mathrm{A} \triangleq \frac{\lambda \sigma_
{1}+\frac{1-b^{2}}{\sigma_{1}}}{(1-b) \lambda+\frac{1-b^{2}}{\sigma_{1}}-1}$
in \eqref{eq:eta1}.

\begin{align*}
\text { Denominator of }\mathrm{A}&= (1-b)\lambda +\frac{1-b^{2}}{\sigma_{1}}-1 \\
& =1-\frac{q}{2(1+c)}+\frac{q\left(1-\tau\right)}{(1+c)\left(1-\tau\right)}-1 \\
& =\frac{q}{2(1+c)}\\
%\end{align*}
%\begin{align*}
\text{Numerator of } \mathrm{A} &= \lambda \sigma_{1}+\frac{1-b^{2}}{\sigma_{1}} \\
& =\frac{(1+c)\left(1-\tau\right)}{1-b}\left(1-\frac{q}{2(1+c)}\right)+\frac{q}{1+c} \\
& =\frac{(1+c)\left(1+b\right)}{q}\left(1-\frac{q}{2(1+c)}\right)+\frac{q}{1+c} \\
& =\frac{(1+c)\left(1+b\right)}{q}+\frac{q}{1+c}\left(1-\frac{(1+c)\left(1+b\right)}{2q}\right) 
\end{align*}
Thus, we get
\begin{align*}
A&=\frac{2(1+b)(1+c)^2}{q^2}+2-\frac{(1+c)\left(1+b\right)}{q} 
\end{align*}
and, 
\begin{align*}
	1-A=\frac{(1+c)\left(1+b\right)}{q} - 1 - \frac{2(1+b)(1+c)^2}{q^2} .
\end{align*}

By substituting $(1-A)$ in Equation~\eqref{eq:eta1} and plugging back $\eta$
into Equation~\eqref{eq:start_generic_proof}, we get
\begin{align}
\expect\left(\left\|u_{\step+1}-u^*\right\|^{2} \mid \mathcal{F}_{0}\right) &\leq \chi^{\step} D+q\left(\chi^{\step}-\frac{1}{\chi}\right)(1-A) \notag\\
&= \chi^{\step}(D+q(1-A))-\frac{1}{\chi} q(1-A) \notag\\
&= \chi^{\step}\left(D+ {(1+c)\left(1+b\right)} - q - \frac{2(1+b)(1+c)^2}{q}\right) \notag\\ &~~~~+\frac{1}{\chi}\left(-{(1+c)\left(1+b\right)} + q + \frac{2(1+b)(1+c)^2}{q}\right) \notag\\
&\leq \chi^{\step}\left(D+ (1+c)(1+b)\right)+\frac{1}{\chi}\left(q + \frac{2(1+b)(1+c)^2}{q}\right) \notag\\
&\leq \chi^{\step}\left(D+ 2(1+c)\right)+\frac{1}{\chi}\left(q + \frac{2(1+b)(1+c)^2}{q}\right). \label{eq:simplified_eta}
\end{align}

\paragraph{Step 4: Upper \& lower bounding $\chi$.}
We can rewrite $\chi$ as follows:
\begin{align*}
\chi &=1+\lambda \left(\sigma_{1}-\left(1-b^{2}\right)\right)-\lambda^{2} \sigma_{1}(1-b) \\
&=1+\lambda\left((1+c)\left(1-\tau\right)-q\left(1-\tau\right)\right)-\lambda^{2}(1+c)\left(1-\tau\right)(1-b) \\
&=1+\lambda\left(1-\tau\right)(1+c-q)-\lambda^{2}(1+c)\left(1-\tau\right)(1-b) \\
&=1+\frac{\left(1-\tau\right)(1+c-q)(1+c-q / 2)}{(1+c)(1-b)}-\frac{\left(1-\tau\right)(1+c-q/ 2)^{2}}{(1+c)(1-b)} \\
&=1-\frac{q\left(1-\tau\right)(1+c-q / 2)}{2(1-b)(1+c)} \\
& =1-\frac{(1+b)(1+c-q / 2)}{2(1+c)}
\end{align*}

\textit{Lower bound:} %$\chi > 1-\frac{1+c-q / 2}{1+c} =\frac{q}{2(1+c)}$.
\begin{align*}
	\chi =1-\frac{1+b}{2} +\frac{q(1+b)}{4(1+c)} > \frac{1-b}{2} = \frac{1-\sqrt{1-q(1-\tau)}}{2} \geq \frac{q(1-\tau)}{4}
\end{align*}
The inequality holds due to the fact that $b=\sqrt{1-q\left(1-\tau\right)} < 1-\frac{q(1-\tau)}{2}$.

\textit{Upper bound:} 
\begin{align*}
\chi=1-\frac{(1+b)}{2}\left(1-\frac{q}{2(1+c)}\right) =\frac{1-b}{2}+\frac{q(1+b)}{4(1+c)}
&\leqslant \frac{1}{2} + \frac{q(1+b)}{4}\\
&\leqslant \frac{1}{2} + \frac{q}{4}\left(2- \frac{q(1-\tau)}{2}\right)\\
&\leqslant 1 - \frac{q^2(1-\tau)}{8}.
\end{align*}
The first inequality holds for any non-negative $b, c, q$.
The second inequality leverages the fact that $b = \sqrt{1-q(1-\tau)} \leq 1-\frac{q(1-\tau)}{2}$.
The final inequality follows from the fact that $ q \in (0,1]$.

\paragraph{Step 5: Final touch.}
By substituting upper and lower bounds of $\chi$ in Equation~\eqref{eq:simplified_eta}, we get
\begin{align*}
\expect\left(\left\|u_{\step+1}-u^*\right\|^{2} \mid \mathcal{F}_{0}\right)  &< \left(1 - \frac{q^2(1-\tau)}{8}\right)^{\step}(D+ 2(1+c))+\frac{4}{q(1-\tau)}\left(q + \frac{2(1+b)(1+c)^2}{q}\right)\\
& =\left(1 - \frac{q^2(1-\tau)}{8}\right)^{\step}\left(D+\frac{2 \sigma_1}{\left(1-\tau\right)}\right)+\frac{4}{q\left(1-\tau\right)}\left(q+\frac{2(1+b) \sigma_1^2}{{q}(1-\tau)^2}\right) \\
& =\left(1 - \frac{q^2(1-\tau)}{8}\right)^{\step}\left(D+\frac{2 \sigma_1}{\left(1-\tau\right)}\right)
+\frac{4}{\left(1-\tau\right)}+\frac{8(1+b) \sigma_1^2}{{q}^2(1-\tau)^3}\\
%& =\left(1 - \frac{q^2(1-\tau)}{8}\right)^{\step}\left(D+\frac{2 \sigma_1}{\left(1-\tau\right)}\right)+\left(\frac{2 \sigma_1}{\left(1-\tau\right)}+\frac{4\sigma_1^3}{{q}(1-\tau)^2(1-b)}\right) \\
& \leq \left(1 - \frac{q^2(1-\tau)}{8}\right)^{\step}\left(D+\frac{2\sigma_1}{\left(1-\tau\right)}\right)
+\left(\frac{4}{\left(1-\tau\right)}+\frac{16\sigma_1^2}{{q}^2(1-\tau)^3}\right) \\
&\leq \left(1 - \frac{q^2(1-\tau)}{8}\right)^{\step}\left(D+\frac{2(\sigma\sqrt{p} + \zeta)}{\sqrt{q}\left(1-\tau\right)}\right)
+\left(\frac{4}{\left(1-\tau\right)}+\frac{8(p\sigma^2+ \zeta^2)}{{q}^3(1-\tau)^3}\right)\\
&\leq\left(1-\frac{q^2(1-\tau)}{8}\right)^{\step}D
+\left(\frac{8 (\sigma\sqrt{p}+ \zeta)}{\sqrt{q}\left(1-\tau\right)}+\frac{8(p\sigma^2+ \zeta^2)}{{q}^3(1-\tau)^3}\right)
\end{align*}
%
%mistakes of past
%\begin{align*}
%\expect\left(\left\|u_{\step+1}-u^*\right\|^{2} \mid \mathcal{F}_{0}\right)  & \leq\left(1 - \frac{q^2(1-\tau)}{8}\right)^{\step}(D+ 2(1+c))+\frac{2(1+c)}{q}\left(q + \frac{2(1+b)(1+c)^2}{q}\right)\\
%& =\left(1 - \frac{q^2(1-\tau)}{8}\right)^{\step}\left(D+\frac{2 \sigma_1}{\left(1-\tau\right)}\right)+\frac{2 \sigma_1}{q\left(1-\tau\right)}\left(q+\frac{2(1+b) \sigma_1^2}{{q}(1-\tau)^2}\right) \\
%& =\left(1 - \frac{q^2(1-\tau)}{8}\right)^{\step}\left(D+\frac{2 \sigma_1}{\left(1-\tau\right)}\right)
%+\frac{2 \sigma_1}{\left(1-\tau\right)}+\frac{4 (1+b) \sigma_1^3}{{q}^2(1-\tau)^3}\\
%%& =\left(1 - \frac{q^2(1-\tau)}{8}\right)^{\step}\left(D+\frac{2 \sigma_1}{\left(1-\tau\right)}\right)+\left(\frac{2 \sigma_1}{\left(1-\tau\right)}+\frac{4\sigma_1^3}{{q}(1-\tau)^2(1-b)}\right) \\
%& \leq \left(1 - \frac{q^2(1-\tau)}{8}\right)^{\step}\left(D+\frac{2\sigma_1}{\left(1-\tau\right)}\right)
%+\left(\frac{2 \sigma_1}{\left(1-\tau\right)}+\frac{8\sigma_1^3}{{q}^2(1-\tau)^3}\right) \\
%&= \left(1 - \frac{q^2(1-\tau)}{8}\right)^{\step}\left(D+\frac{2(\sigma\sqrt{p} + \zeta)}{\sqrt{q}\left(1-\tau\right)}\right)
%+\left(\frac{2 (\sigma\sqrt{p}+ \zeta)}{\sqrt{q}\left(1-\tau\right)}+\frac{8(\sigma\sqrt{p}+ \zeta)^3}{{q}^{7/2}(1-\tau)^3}\right)
%\end{align*}
%This concludes the proof.
%We further observe that $1- b = 1 - \sqrt{1-q(1-\tau)} = 1 - \sqrt{1-q+q\tau} \leq 1- \sqrt{q\tau}$. 

We can also alternatively write the result as follows. Since $(1-a)^{\step}
\leq \exp(-ak)$ for $a\in [0,1)$ and $\step \in \mathbb{N}$, we have:
\begin{align*}
\expect\left(\left\|u_{\step+1}-u^*\right\|^{2} \mid \mathcal{F}_{0}\right) %&\leqslant\left(\frac{1-\sqrt{q\tau}}{2}\right)^{\step}\left(D+\frac{2(\sigma+ \zeta)}{\sqrt{q}\left(1-\tau\right)}\right)+\left(\frac{2 (\sigma+ \zeta)}{\sqrt{q}\left(1-\tau\right)}+\frac{8(\sigma+ \zeta)^3}{{q}^{7/2}(1-\tau)^3}\right)\\
%&\leqslant\exp\left(-\frac{q^2(1-\tau)}{8}{\step}\right)\left(D+\frac{2(\sigma\sqrt{p}+ \zeta)}{\sqrt{q}\left(1-\tau\right)}\right)+\frac{8(p\sigma^2+ \zeta^2)}{{q}^3(1-\tau)^3}+\frac{4}{\left(1-\tau\right)}\\
&\leqslant \exp\left(-\frac{q^2(1-\tau)}{8}{\step}\right)D
+\left(\frac{8 (\sigma\sqrt{p}+ \zeta)}{\sqrt{q}\left(1-\tau\right)}+\frac{8(p\sigma^2+ \zeta^2)}{{q}^3(1-\tau)^3}\right)
\end{align*}

%We can avoid the $\sigma^3$ at the expense of a constant in bounding $1-b$ in the following:
%
%\begin{align*}
%	\expect[\left\|u_{\step+1}-u^*\right\|^{2}]  \leqslant &\left(1 - \frac{q^2(1-\tau)}{8}\right)^{\step}\left(D+\frac{2(\sigma\sqrt{p} + \zeta)}{\sqrt{q}\left(1-\tau\right)}\right)+\frac{8(p\sigma^2+ \zeta^2)}{{q}^3(1-\tau)^3}+\frac{4}{\left(1-\tau\right)} \\
%%&\leq\left(1-\frac{q^2(1-\tau)}{8}\right)^{\step}\left(D+\frac{2(\sigma\sqrt{p}+ \zeta)}{\sqrt{q}\left(1-\tau\right)}\right)
%%		+\left(\frac{2 (\sigma\sqrt{p}+ \zeta)}{\sqrt{q}\left(1-\tau\right)}+\frac{8(\sigma\sqrt{p}+ \zeta)^3}{{q}^{7/2}(1-\tau)^3}\right)\\
%	   &\leq\left(1-\frac{q^2(1-\tau)}{8}\right)^{\step}D
%	   +\left(\frac{8 (\sigma\sqrt{p}+ \zeta)}{\sqrt{q}\left(1-\tau\right)}+\frac{8(p\sigma^2+ \zeta^2)}{{q}^3(1-\tau)^3}\right)
%   \end{align*}

\end{proof}

\subsection{Technical Lemma on the Learning Rate}

We prove below a technical lemma used in the proof of Theorem~\ref{thm:utility_general}.

\begin{lemma}[Choices of the Learning Rate]\label{lemma:lrate}
	In order to ensure convergence of Algorithm~\ref{algo:pfix}, we should choose the learning rate $\lambda$ in the range
	\[\left(\frac{1+c-q}{(1+c)(1-b)}, \frac{1+c-q}{(1+c)(1-b)}\left(\frac{1}{2}+  \frac{1}{2}\sqrt{1 + 4 \frac{(1+c)(1-b)}{(1-\tau)(1+c-q)^2}}\right)\right).\]
	Here, we assume that there exists $c > 0$ such that $\sigma_1 \triangleq 
	\frac{\sigma\sqrt{p} + \zeta}{\sqrt{q}} \triangleq (1+c)(1-\tau)$, $b \triangleq 
	\sqrt{1- q(1-\tau)}$, $\sigma$ is the noise variance, $\tau\in[0,1)$ is
	the contraction factor, and $q \in (0,1]$.
\end{lemma}
\begin{proof}
	In order to ensure convergence of the algorithm, we need to satisfy
	$0 < \chi < 1$. We observe that
	\[ \chi = 1+\lambda\left(\sigma_{1}-\left(1-b^{2}\right)\right)-\lambda^{2} \sigma_{1}(1-b),\]
	 As $\chi$ is a function of the learning rate, the upper and lower bounds on $\chi$ impose lower and upper bounds on the desired learning rate $\lambda$.
	
	\textbf{Step 1: Lower Bounding $\lambda$.} 
	\begin{align*}
		\chi < 1 &\implies 1+\lambda\left(\sigma_{1}-\left(1-b^{2}\right)\right)-\lambda^{2} \sigma_{1}(1-b) < 1\\
		&\underset{(a)}{\implies} \left(\sigma_{1}-\left(1-b^{2}\right)\right)-\lambda \sigma_{1}(1-b) < 0\\
		&\implies \frac{\sigma_{1}-\left(1-b^{2}\right)}{\sigma_{1}(1-b)} < \lambda\\
		&\implies \frac{(1+c)(1-\tau)-q \left(1-\tau^{2}\right)}{(1+c)(1-\tau)(1-b)} < \lambda\\
		&\implies \frac{1+c-q}{(1+c)(1-b)} < \lambda
	\end{align*} 
	Step (a) holds true for $\lambda > 0$, i.e. for any positive learning rate.
	
	\textbf{Step 2: Upper Bounding $\lambda$.} As $\chi > 0$, we should choose
	the learning rate $\lambda$ in a range such that the following quadratic equation satisfies
	\begin{align*}
		1+\lambda\left(\sigma_{1}-\left(1-b^{2}\right)\right)-\lambda^{2} \sigma_{1}(1-b) > 0.
	\end{align*}
	Since the coefficient corresponding to $\lambda^2$ is negative, the quadratic equation stays positive only between its two roots:
	\[\lambda_{inf} = \dfrac{(\sigma_{1}-(1-b^2))  - \sqrt{(\sigma_{1}-(1-b^2))^2 + 4 \sigma_{1}(1-b)}}{2 \sigma_{1}(1-b)} ,\]
	and
	\[\lambda_{sup} = \dfrac{(\sigma_{1}-(1-b^2))  + \sqrt{(\sigma_{1}-(1-b^2))^2 + 4 \sigma_{1}(1-b)}}{2 \sigma_{1}(1-b)} .\]
	Since the smallest root $\lambda_{inf}$ is negative, and we care about
	only positive learning rates, it provides a vacuous bound. Thus, we can
	ignore it.
	
	Thus, we conclude that
	\begin{align*}
		\lambda &< \dfrac{(\sigma_{1}-(1-b^2))  + \sqrt{(\sigma_{1}-(1-b^2))^2 + 4 \sigma_{1}(1-b)}}{2 \sigma_{1}(1-b)}\\
		&= \dfrac{(1+c-q)(1-\tau)  + \sqrt{(1+c-q)^2(1-\tau)^2 + 4 (1+c)(1-b)(1-\tau)}}{2 (1+c)(1-\tau)(1-b)}\\
		&= \dfrac{(1+c-q)  + (1+c-q) \sqrt{1 + 4 \frac{(1+c)(1-b)}{(1-\tau)(1+c-q)^2}}}{2 (1+c)(1-b)}\\
		&= \frac{1+c-q}{(1+c)(1-b)}\left(\frac{1}{2}+  \frac{1}{2}\sqrt{1 + 4 \frac{(1+c)(1-b)}{(1-\tau)(1+c-q)^2}}\right)
	\end{align*}
    to obtain a valid convergence of the algorithm.
\end{proof}

% !TEX root = main.tex

\section{Derivation of Private ADMM Updates}
\label{app:algos}

In this section, we give details on how to obtain the private ADMM updates
given in Algorithm~\ref{algo:conADMM}, Algorithm~\ref{algo:p2ppADMM}
and Algorithm~\ref{algo:fedADMM} from our general noisy fixed-point iteration 
(Algorithm~\ref{algo:pfix}).

\subsection{Warm-up: Non-Private ADMM}
\label{app:non-private-admm}

For clarity and self-completeness, we start by deriving the standard ADMM
updates from the fixed-point iteration formulation described in
Section~\ref{sec:admm_fixed}. This derivation follows the lines of
\citep[][Appendix~B therein]{Giselsson2016LineSF}.

Recall that ADMM solves an optimization problem of the form 
\eqref{eq:genericADMM}, which we restate here for convenience:

\begin{equation}
    \begin{array}{ll}
\min _{x, z} & f(x)+g(z) \\
s . t . & A x+B z=c
\end{array}
\label{eq:repeat_ADMM_pb}
\end{equation}

We also recall the definition of the infimal postcomposition.

\begin{defn}[Infimal postcomposition]
    Let $M$ be a linear operator. The infimal postcomposition $M
    \triangleright f$ is defined by
\begin{equation*}
    (M \triangleright f)(y)=\inf \{f(x) \mid M x=y\}.
\end{equation*}
\end{defn}
  
As mentioned in Section~\ref{sec:admm_fixed}, the minimization problem above
can be rewritten as
$$
\min _{u}(-A \triangleright f)(-u-c)+(-B \triangleright g)(u).
$$

Introducing $p_{1}(u)=(-A \triangleright f)(-u-c)$ and $p_{2}(u)=(-B
\triangleright g)(u)$ recovers a minimization problem solvable with the 
Douglas-Rachford algorithm. Formally, the $\lambda$-averaged ADMM can be
written as the following fixed-point operator:
\begin{equation}
\label{eq:non-private-ADMM}
u_{k+1}=u_{k}+\lambda\left(R_{\gamma p_{1}} (R_{\gamma p_{2}}\left(u_{k})\right)-u^{k}\right),
\end{equation}
where $R_{\gamma p_{1}}=2\prox_{\gamma p_1} - I$ and $R_{\gamma p_{2}}=2\prox_
{\gamma p_2} - I$.

From this generic formula, we can recover the standard ADMM updates in terms
of $x$ and $z$. We start by
rewriting $R_{\gamma p_{2}}(u)$:
$$
\begin{aligned}
 R_{\gamma p_{2}}(u)&=2 \prox_{\gamma p_{2}}(u)-u \\
& =2 \argmin_v\left\{\inf _{z}\{g(z) \mid-B z=v\}+\frac{1}{2 \gamma}\|u-v\|^
{2}\right\}-u \\
& =-2 B \argmin_z\left\{g(z)+\frac{1}{2 \gamma}\|B z+u\|^{2}\right\}-u.
\end{aligned}
$$

This leads to the introduction of the $z$ variable with associated update:
\begin{equation}
\label{eq:non-private-admm-z}
z_{k+1}=\argmin_z\left\{g(z)+\frac{1}{2 \gamma}\|B z+u_k\|^{2}\right\}.
\end{equation}

Similarly, we can rewrite $R_{\gamma p_1}$:
$$
\begin{aligned}
    R_{\gamma p_1}(u) & =2 \prox_{\gamma p_{1}}(u)-u \\
    &=2\argmin_v\left\{\inf_x \{f(x) \mid-A x=-v-c\}+
    \frac{1}
    {2 \gamma}\|u-v\|^{2}\right\} -u \\
& =2A \argmin_{x}\left\{f(x)+\frac{1}{2 \gamma}\|A x-u-c\|^{2}\right\}-2c-u,
\end{aligned}
$$
which leads to the introduction of the $x$ variable with associated update:
\begin{equation}
\label{eq:non-private-admm-x}
x_{k+1}=\argmin_x\left\{f(x)+\frac{1}{2 \gamma}\left\|A x+2 B z_{k+1}+u_{k}-c\right\|^{2}\right\}.
\end{equation}
Based on \eqref{eq:non-private-admm-z} and \eqref{eq:non-private-admm-x}, we
can rewrite:
\begin{equation*}
\begin{aligned}
    R_{\gamma p_1}R_{\gamma p_2}\left(u_{k}\right) & =R_{\gamma p_1}\left(-2 B
    z_{k+1}-u_k\right) \\
& =2 A x_{k+1}-2 c-\left(-2 B z_{k+1}-u_{k}\right) \\
& =2\left(A x_{k+1}+B z_{k+1}-c\right)+u_{k},
\end{aligned}
\end{equation*}
which in turns gives for the update of variable $u$ in 
\eqref{eq:non-private-ADMM}:
\begin{equation}
\label{eq:non-private-admm-u}
\begin{aligned}
    u_{k+1}=u_{k}+2\lambda\left(A x_{k+1}+B z_{k+1}-c\right),
\end{aligned}
\end{equation}

The updates \eqref{eq:non-private-admm-z}, \eqref{eq:non-private-admm-x} and 
\eqref{eq:non-private-admm-u} correspond to the standard ADMM updates 
\citep{admmbook,Giselsson2016LineSF}.

\subsection{General Private ADMM}
\label{app:general-private-admm}

We now introduce a general private version of ADMM to solve problem 
\eqref{eq:repeat_ADMM_pb}. In this generic part, we consider without loss of
generality that the data-dependent part is in the function $f$. For
clarity, we denote $f(x)$ by $f(x;\cD)$ to make the dependence on the
dataset $\cD$ explicit.

Following our general noisy fixed-point iteration
(Algorithm~\ref{algo:pfix}), the private counterpart of the non-private ADMM
iteration \eqref{eq:non-private-ADMM} is given by:
$$
u_{k+1}=u_{k}+\lambda\left(R_{\gamma p_{1}}( R_{\gamma p_{2}}\left(u_
{k});\cD\right)-u_{k}+\eta_{k+1}\right),
$$
where the notation $R_{\gamma p_{1}}(\cdot;D)$ is again to underline the
data-dependent part of the computation.

By following the same derivations as in Appendix~\ref{app:non-private-admm},
we obtain the following equivalent update:
$$
u_{k+1}=u_{k}+2 \lambda\left(A x_{k+1}+B z_{k+1}-c+\frac{1}{2} \eta_
{k+1}\right),
$$
where $z_{k+1}$ and $x_{k+1}$ are defined as in \eqref{eq:non-private-admm-z}
and \eqref{eq:non-private-admm-x} respectively. The full algorithm is given in
Algorithm~\ref{algo:pADMM}. Note that we return only $z_K$, which is
differentially private by postprocessing of $u_{K-1}$ (see
Appendix~\ref{app:privacy}). In contrast, returning $x_K$ would violate
differential privacy as the last update interacts with the data without
subsequent random perturbation. In many problems (such as the consensus problem
considered below), returning $z_K$ is sufficient for all practical purposes.
Note that when $A$ is invertible (which is the case for consensus, see below),
one can recover from $z_K$ the unique $
\tilde{x}_K=A^{-1}(c-Bz_K)$ such that $(\tilde{x}_K, z_K)$ satisfies the
constraint in problem \eqref{eq:repeat_ADMM_pb}.

\begin{algorithm}[t]
    \SetKwComment{Comment}{$\triangleright$ }{}
    \DontPrintSemicolon
    \KwIn{initial point
    $u_0$, step size $\lambda\in(0,1]$, privacy noise variance
    $\sigma^2\geq 0$, Lagrange parameter $\gamma>0$}
    \For{$k=0$ to $K-1$}{
        $z_{k+1}=\argmin_z\left\{g(z)+\frac{1}{2 \gamma}\|B z+u_k\|^{2}\right\} $ \;
        $x_{k+1}=\argmin_x\left\{f(x;\cD)+\frac{1}{2 \gamma}\left\|A x+2 B z_
        {k+1}+u_{k}-c\right\|^{2}\right\}$\;
        $ u_{k+1}=u_{k}+2\lambda\left(A x_{k+1}+B z_{k+1}-c + \frac{1}{2}\eta_{k+1} \right) \quad\text{with } \eta_{k+1}\sim\gau{ \sigma^2\mathbb{I}}$\;
    }
    \Return{$z^{K}$}
    \caption{General Private ADMM to solve problem \eqref{eq:repeat_ADMM_pb}}
    \label{algo:pADMM}
\end{algorithm}

\subsection{Instantiations for the Consensus Problem}

We now instantiate the generic private ADMM update given in
Appendix~\ref{app:general-private-admm} to the consensus problem and derive
centralized, fully decentralized and federated private ADMM algorithms for
ERM.

Recall that the ERM problem \eqref{eq:simple_erm} can be reformulated as the
consensus problem \eqref{eq:admm_consensus}, which we restate below for
convenience:
$$
\begin{aligned}
& \min_{x\in\mathbb{R}^{np},z\in\mathbb{R}^p} \frac{1}{n} \sum_{i=1}^{n}
f\left(x_
{i};d_i\right)+r(z) \\
& \text { s.t } x_{i}=z\quad \forall i,
\end{aligned}
$$
which is a special case of problem \eqref{eq:repeat_ADMM_pb} with $x=
(x_1,\dots,x_n)^\top$ composed of $n$ blocks of $p$ coordinates, $f(x)=f
(x;\cD)=\frac{1}{n}\sum_{i=1}^nf(x_i;d_i)$, $g(z)=r(z)$, $c=0$, $A=I$ and
$B=-I_{n(p\times
p)}\in\R^{n\times p}$ where $I_{n(p\times p)}\in\R^{n\times p}$ denotes $n$
stacked identity matrices of size $p\times p$.

\paragraph{Centralized private ADMM (Algorithm~\ref{algo:conADMM}).}
We use the specific structure of the consensus problem to simply the general
private ADMM updates in Appendix~\ref{app:general-private-admm}.
The $z$-update gives:
$$
\begin{aligned}
& z_{k+1}=\underset{z}{\argmin}\Bigg\{r(z)+\frac{1}{2
\gamma}\bigg\|\bigg(\begin{array}{l}
    I \\
    \dots\\
    I
    \end{array}\bigg)z-u_{k}\bigg\|^{2}\Bigg\}, \\
& z_{k+1}=\prox_{\gamma r}\Big(\frac{1}{n} \sum_{i=1}^n u_{k,i}\Big).
\end{aligned}
$$

For the $x$-update, we have:
$$
x_{k+1}  =\underset{x}{\argmin}\Bigg\{f(x;\cD)+\frac{1}{2 \gamma}\bigg\| x-2
\bigg(\begin{array}{l}
    I \\
    \dots\\
    I
    \end{array}\bigg) z_{k+1}+u_{k}\bigg\|^{2}\Bigg\}.
$$
As $f$ is fully separable, this can be decomposed into $n$ block-wise
updates as:
\begin{align*}
    x_{k+1,i} &= \underset{x_i}{\argmin}\left\{f(x_i;d_i) +\frac{1}{2\gamma} 
    \norm{x_i - 2z_{k+1}+u_{k,i}}\right\}\\
    &= \prox_{\gamma f_i}(2z_{k+1} - u_{k,i}).
\end{align*}

Finally, the $u$-update writes:
$$
u_{k+1}=u_{k}+2 \lambda\Bigg( x_{k+1}-\bigg(\begin{array}{c}
    I \\
    \dots \\
    I
    \end{array}\bigg) z_{k+1}+\frac{1}{2} \eta_{k+1}\Bigg),
$$
which can be equivalently written as block-wise updates:
$$
u_{k+1,i}=u_{k,i}+2 \lambda\left(x_{k+1,i}-z_{k+1}+\frac{1}{2} \eta_{k+1,i}\right).
$$

Algorithm~\ref{algo:conADMM} shows the resulting algorithm when cycling over
the $n$ blocks of $x$ and $u$ in lexical order (which is equivalent to
considering a single block, i.e., $B=1$). But remarkably, the flexibility of
our general noisy fixed-point iteration (Algorithm~\ref{algo:pfix}) and
associated utility result (Theorem~\ref{thm:utility_general}) allows us to
cover many other interesting cases, some of which directly leading to
federated and fully decentralized learning algorithms (see below). In
particular, we can sample the blocks in a variety of ways, such as:
\begin{enumerate}
    \item cycling over an independently chosen random permutation of the
    blocks at each iteration (the corresponding utility
can be obtained by setting $q=1$ in Theorem~\ref{thm:utility_general});
    \item choosing a single random block at each iteration $k$ (this is
    used to obtained our fully decentralized algorithm);
    \item choosing a random subset of $m$ blocks (this is used to obtain
    our federated algorithm with user sampling).
\end{enumerate}
The utility guarantees can be obtained from
Theorem~\ref{thm:utility_general} by setting $q=1$ in case 1, $q=1/n$ in case
2, and $q=m$ in case 3.

% Algorithm 2: Private consensus ADMM

% $$
% \begin{aligned}
% & \hat{z}^{k+1}=\frac{1}{n} \sum_{i=1}^{n}u_i^k \\
% &z^{k+1}=\prox_{\gamma r} \left(\hat{z}^{k+1}\right) \\
% & \text { for } i=1 \text { to } n \text { do } \\
% & x_i^{k+1} = \prox_{\gamma f_i}(2z^{k+1} - u_i^k) \\
% & u_{i}^{k+1}=u_{i}^{k}+2 \lambda\left(x_{i}^{k+1}-z^{k+1}+\frac{1}{2} \eta_{i}^{k+1}\right)
% \end{aligned}
% $$

\paragraph{Federated private ADMM (Algorithm~\ref{algo:fedADMM}).} Our
federated private ADMM algorithm exactly mimics the updates of centralized
private ADMM (Algorithm~\ref{algo:conADMM}), which can be executed in a
federated fashion since (i) the blocks $x_i$ and $u_i$ associated to each user
$i$ can be updated in parallel by each user, and (ii) if each user $i$ shares
$u_{k+1,i}-u_{k,i}$ with the server, then the latter can execute the rest of the
updates. The more general version with user sampling given in
Algorithm~\ref{algo:fedADMM} is obtained by choosing a random subset of $m$
blocks (users) uniformly at random.

% Algorithm 4: Federated ADMM

% Server loop

% for $k=1$ to $k$ do

% Sample set of user $W$ with probability of

% Sampling 9

% for $u \in U d o$

% $$
% \begin{aligned}
% & z_{u}=\text { LocalADMMstep}(z) \\
% & \hat{z}^{k+1} = z^k + \frac{1}{n}\sum_{u \in U}z_u\\
% z^{k+1} = \prox_{\gamma r}(\hat{z}^{k+1})
% \end{aligned}
% $$

% LocalADMMstep$(x)$ :

% $$
% \begin{aligned}
%     x_i^{k+1} &= \prox_{\gamma f_i}(2z^{k} - u_i^k)\\
%     u_{i}^{k+1}&=u_{i}^{k}+2 \lambda\left(x_{i}^{k+1}-z^{k}+\frac{1}{2} \eta_{i}^{k+1}\right)\\
% & \text { Returns } u_i^{k+1} - u_i^k
% \end{aligned}
% $$

\paragraph{Fully decentralized private ADMM (Algorithm~\ref{algo:p2ppADMM}).}
In the fully decentralized setting, each user $i$ with local dataset $d_i$ is
associated with blocks $x_i$ and $u_i$.
Our fully decentralized private ADMM algorithm (Algorithm~\ref{algo:p2ppADMM})
directly follows from a block-wise version of Algorithm~\ref{algo:conADMM},
where at each iteration $k$ we select uniformly at random a single block 
(user) to update. This corresponds to a user performing an update on its
local parameters before sending it to another user chosen at random.

% Algorithm: Fully decentralized ADMM

% \begin{align*}
%     x_i^{k+1} &= \prox_{\gamma f_i}(2z^{k} - u_i^k)\\
%     u_{i}^{k+1}&=u_{i}^{k}+2 \lambda\left(x_{i}^{k+1}-z^{k}+\frac{1}{2} \eta_{i}^{k+1}\right)\\
%     \hat{z}^{k+1} &= z^k + \frac{1}{n} (u_i^{k+1}-u_i^k)\\
%     z^{k+1}&=\prox_{\gamma r} \left(\hat{z}^{k+1}\right) \\
% \end{align*}

% Send to next neighbor

% !TEX root = main.tex

\section{Privacy Analysis of our ADMM Algorithms}
\label{app:privacy}

% !TEX root = main.tex

\subsection{Reminders on Privacy Amplification}
\label{app:amp}

In this appendix, we recap known results on privacy amplification that we use
in our own privacy analysis.

\subsubsection{Privacy Amplification by Iteration}
\label{app:iter}

Privacy amplification by iteration refers to the privacy loss decay when only revealing the final output of successive applications of non-expansive operators instead of the full trajectory updates. This was introduced by the seminal work of \citet{ampbyiteration}, later extended by \citet{Altschuler2022}. 

We recap here the main theorem which characterizes the privacy loss of a given contribution in an algorithm defined as the sequential applications of non-expansive operators.

\begin{thm}[Privacy amplification by iteration~\cite{ampbyiteration}]
    \label{thm:amp_iter}
        Let $T_{1}, \dots, T_{K}, T'_{1}, \dots, T'_{K}$ be non-expansive
        operators, $U_{0}\in\cU$ be an initial random state, and $(\zeta_{k})_
        {k=1}^K$ be a
        sequence
        of noise
        distributions. Now, consider the noisy iterations $U_{k+1}\triangleq T_{k+1}(U_k)+\eta_
        {k+1}$ and
        $U'_{k+1}\triangleq T_{k+1}(U'_k)+\eta'_{k+1}$, where $\eta_{k+1}$ and $\eta_{k+1}'$ are drawn independently from distribution $\zeta_{k+1}$.
        Let $s_k \triangleq \sup_{u\in\cU} \norm{T_k (u) - T'_k(u)}$. Let $(a_k)_{k=1}^K$ be a
        sequence of
        real
        numbers such that
        \[\forall k \leq K, \sum_{k' \leq k} s_{k'} \geq \sum_{k' \leq k} a_{k'}, 
        \text{
        and } \sum_{k \leq K} s_k = \sum_{k \leq K} a_k\,.
            \]
        Then,
        \begin{equation}
            D_{\alpha}(U_K || U'_K) \leq \sum_{k=1}^K \sup _{u:\|u\| \leq a_k} D_
            {\alpha}(\zeta_k * \mathbf{u} \| \zeta_k)\,,
        \end{equation}  
        where $*$ is the convolution of probability distributions and $
        \mathbf{u}$ denotes the distribution of the random variable that is always
        equal to $u$.
    \end{thm}

Informally, this theorem allows an amplification factor proportional to the
number of updates performed after the studied step. 
If we compare two
scenarios where only the step $i$ differs by using $d_i$ or $d_i'$, such that revealing this step would lead to a privacy loss $\eps$, it we reveal only step $i+k$, an appropriate choice of $a$ sequence leads to a privacy loss of magnitude $\eps/k$.

\subsubsection{Privacy amplification by subsampling}
\label{app:samp}

When a DP algorithm is executed on a random subsample of data points, and the
choice of this subsampling remains secret, we can obtain privacy
amplification. This privacy amplification by subsampling effect has been
extensively
studied under
various sampling schemes \citep{Balle_subsampling,
amp_sub_mironov} and is classically used in the privacy analysis of
DP-SGD \citep{bassily2014Private,abadi2016Deep,Altschuler2022}. While tighter
bounds can be computed numerically, here for the sake of simplicity we use a
simple closed-form expression which gives the order of magnitude of the
amplification.

\begin{lemma}[Amplification by subsampling, \citealp{Altschuler2022}]
    Let $q<1/5$, $\alpha>1$ and $\sigma \geq 4$. Then, for $\alpha \leq \left
    (M^2 \sigma^2 / 2-\log \left(5 \sigma^2\right)\right) /\left(M+\log (q
    \alpha)+1 /\left(2 \sigma^2\right)\right)$ where $M=\log (1+1 /(q
    (\alpha-1)))$, the subsampled Gaussian mechanism with probability $q$ and
    noise parameter $\sigma^2$ satisfies $(\alpha,\eps_{samp})$-RDP with
    \begin{equation*}
        \eps_{samp} \leq \frac{2\alpha q^2 \Delta^2}{\sigma^2}.
    \end{equation*}
\end{lemma}

\subsection{Sensitivity Bounds}

We aim at bounding the privacy loss of the general centralized
ADMM introduced in Section~\ref{app:non-private-admm}. We assume that $K$
iterations are done with only $f$ interacting with data, i.e., the
data-dependent step lies in the $x$-update.
We assume that all data points are used with uniform weighting, meaning that
$f$ can be written as $f(x;\cD) = \frac{1}{n}\sum_{i=1}^n f(x;d_i)$.

To bound the privacy loss, we aim at computing the Rényi divergence between
the distribution of the outputs, which can be linked to the sensitivity of the
fixed-point update \eqref{eq:non-private-admm-u} to the change of one data
point. For any pair of neighboring datasets $\cD \sim \cD'$ that differs only
on data item $d_i$ (i.e., $d_j \neq d_j'
\implies i=j$) and any $u$, we thus want to
bound
the difference between $T(u)$ computed on dataset $\cD$ and $T'(u)$ computed
on the dataset $\cD'$. We note $x$ (resp. $x'$) and $z$ (resp. $z'$) the primal variables in the calculation.

We first investigate how the sensitivity of the data-dependent update
propagates to $u$.
As only $x$-updates are data-dependent, the $z$ stays identical for $\cD$ and
$\cD'$ and thus we have:
\begin{equation}
\label{eq:T-minus-Tp}
    T(u) - T'(u) = 2 \lambda A (x - x').
\end{equation}

We bound the sensitivity by assuming that $A$ has its smallest singular value
$\omega_{A} >0$.

Let us define $\varphi(x) = \frac{1}{2\gamma} \norm{Ax +Bz +u +c}^2$.
$\varphi$ is twice differentiable and we have
\begin{align*}
    \nabla_{x} \varphi(x) & =\frac{1}{2 \gamma} \nabla_{x}\left(x^{\top} A^{\top}
    A x-2 (Bz +u +c)^{\top} A x+(Bz +u +c)^{\top} (Bz +u +c)\right) \\
    & =\frac{1}{2 \gamma}\left(2 A^{\top} A x-2 A (Bz +u +c)\right), \\
    \nabla_{x}^2 \varphi(x) & =\frac{1}{\gamma} A^{\top} A.
\end{align*}

Thus, $\varphi$ is $\mu$-strongly convex if and only:
$$
\mu I_{n} \preceq  \frac{1}{\gamma} A^{T} A.
$$

This is satisfied when the smallest eigenvalue of $A^{\top} A$ is larger than
$\mu$. This corresponds to the same condition on the smallest singular value
$\omega_{A}$ of $A$, hence
$$
\omega_{A} \geq  \mu \gamma,
$$
and thus $\varphi$ is $\frac{\omega_{A}}{\gamma}$-strongly convex.

Let us now consider $F(x)=f(x ; \cD)+\varphi(x)$ and $F'(x) = f'
(x;\cD')+\varphi(x)$. We assume that $f_i(\cdot)=f(\cdot;d_i)$ are convex,
differentiable and
$L$-Lipschitz with respect to the $l_{2}$ norm for all possible $d$.
Then, using a classic result on the sensitivity of the $\argmin$ of strongly
convex functions \cite{chaudhuri2011Differentially}, the sensitivity of
$\argmin F(x)$ is bounded by:
$$
\left\|x-x^{\prime}\right\| \leqslant \frac{2 L \gamma}{n \omega_{A}}.
$$

Finally, by re-injecting this formula into \eqref{eq:T-minus-Tp}, we get the
final bound:
\begin{equation}
\label{eq:sens_general_bound}
    \norm{T(u) - T'(u)} \leq \frac{4\lambda L\gamma \norm{A}_2}{n \omega_{A
    }}.
\end{equation}

\paragraph{Special case of the consensus problem.}

In the case of the consensus problem, we can derive a tighter upper bound for
the sensitivity of the block-wise update for which the data point is
different between $\cD$ and $\cD'$:
\begin{equation*}
    T(u)_i - T'(u)_i = 2 \lambda(x_i-x_i').
\end{equation*}

In this case, the $x_i$ can be simply rewritten as $\prox_{\gamma f_i}(2z-u)$,
where $f_i$  is $L$-Lipschitz, and we have:
\begin{equation*}
    \norm{x_i - x_i'} \leq 2 L \gamma.
\end{equation*}
Therefore, $\norm{ T(u)_i - T'(u)_i} \leq 4\lambda L\gamma$ and then
\begin{equation}
\label{eq:sens_consensus_bound}
   \norm{ T(u) - T'(u)} \leq \frac{4\lambda L\gamma}{n}.
\end{equation}

\subsection{General Centralized Private ADMM}

We can now derive the privacy loss of our general private ADMM algorithm 
(\cref{algo:pADMM}).

\begin{thm}[Private classic centralized ADMM]
    Let $A$ be full rank and $\omega_A>0$ the minimal module of its singular
    values. After performing $K$ iterations, \cref{algo:pADMM} is $(\alpha,
    \eps(\alpha))$-RDP with
    \begin{equation}
    \eps(\alpha) = \frac{8 K \alpha \norm{A}_2^2  L^2 \gamma^2}{\sigma^2 n^2
    \omega_{A}^2}.
\end{equation}
\label{thm:pADMMloss}
\end{thm}
\begin{proof}
Recall that the output of the algorithm is $z_{K}$. We also recall that, for a
function of sensitivity $\Delta$, we know that the addition of Gaussian
noise of parameter $\sigma^2$ gives $(\alpha, \alpha \frac{\Delta^2}
{2\sigma^2})$-RDP.

Hence, using the sensitivity bound given in 
\eqref{eq:sens_general_bound}, a single update
leads to a privacy loss of
$$\eps(\alpha) = \frac{8 \alpha \norm{A}_2^2  L^2 \gamma^2}{\sigma^2 n^2
\omega_{A}^2}.$$
We conclude by using by the composition property of
RDP over the $K$ iterations and the robustness to postprocessing.
\end{proof}

Note that the theorem only requires the matrix $A$ to be full rank,
which is a mild assumption.
% as it implies that the constraint on $x$ and $z$
% allows to
% uniquely reconstruct $x$ from $z$.

In particular, for the consensus problem, $A$ is the identity matrix. This
leads to the following privacy guarantee.

\begin{thm}
    After performing $K$ iterations, \cref{algo:conADMM} is $(\alpha,
    \eps)$-RDP with
    \begin{equation*}
        \eps(\alpha) = \frac{8 K \alpha  L^2 \gamma^2}{\sigma^2 n^2}.
    \end{equation*}    
\end{thm}
\begin{proof}
The proof is the same as that of Theorem~\ref{thm:pADMMloss} except that we
use the improved sensitivity bound given in \eqref{eq:sens_consensus_bound}.
\end{proof}

\subsection{Federated Private ADMM with Subsampling}
\label{app:privacy_fl}

As explained in the main text, we can
derive two levels of privacy for the federated algorithm. One is
achieved at the level of users thanks their local injection of noise: this
ensures local DP. The second one in achieved with
respect to a third party observing only the final model: this is central DP.
In the latter case, the local privacy level is amplified by the subsampling of
users and the sensitivity is further reduced by the aggregation step.

We start by the local privacy guarantee.

\begin{thm}[LDP of federated ADMM]
    Let $K_i$ be the number of participations of user $i$. \cref{algo:fedADMM}
    satisfies $(\alpha,\eps_i)$ local RDP for user $i$ with
    \begin{equation*}
        \eps_i \leq \mathcal{O}\left(\frac{8K_i\alpha L^2 \gamma^2}
        {\sigma^2}\right).
    \end{equation*}
\end{thm}
\begin{proof}
We first derive the local privacy loss of sharing $z$.
Using the sensitivity bound \eqref{eq:sens_consensus_bound} derived for the
centralized case and
the fact that we consider a post-processing of $u$, we have
\begin{equation}
    \eps_{loc} \leq \frac{8\alpha L^2 \gamma^2}{\sigma^2}.
    \label{eq:epsloc}
\end{equation}
We obtain the total local privacy loss by composition over the $K_i$
participations of user $i$.
\end{proof}

We now turn to the central privacy guarantee.

\begin{thm}
    Let $m<n/5$, $\alpha>1$ and $\sigma \geq 4$, then for $\alpha \leq \left(M^2 \sigma^2 / 2-\log \left(5 \sigma^2\right)\right) /\left(M+\log (m \alpha/n)+1 /\left(2 \sigma^2\right)\right)$ where $M=\log (1+1 /(m(\alpha-1)/n)).$ Then, \cref{algo:fedADMM} has for central DP loss the following bound:

    \begin{equation*}
        \eps \leq \frac{16K \alpha L^2 \gamma^2}{n^2 \sigma^2}
    \end{equation*}
\end{thm}
\begin{proof}
% We rely on amplification by subsampling as
% defined in .
% \begin{lemma}[Amplification by subsampling \cite{Altschuler2022}]
%     Let $q<1/5$, $\alpha>1$ and $\sigma \geq 4$, then for $\alpha \leq \left(M^2 \sigma^2 / 2-\log \left(5 \sigma^2\right)\right) /\left(M+\log (q \alpha)+1 /\left(2 \sigma^2\right)\right)$ where $M=\log (1+1 /(q(\alpha-1)))$ then the subsampling with probability $q$ of the Gaussian mechanism of parameter $\sigma$ has a Renyi privacy loss bounded by:
%     \begin{equation*}
%         \eps_{samp} \leq \frac{2\alpha q^2 \Delta^2}{\sigma^2}
%     \end{equation*}
% \end{lemma}
Recall that we subsample $m$ participants at each round. By the
reduction of sensitivity due to the aggregation of the $m$ participations,
the initial privacy
loss for one iteration is $\eps_{loc}/m^2$, where $\eps_{loc}$ is
given in \eqref{eq:epsloc}. Then,
applying privacy amplification by subsampling (see \cref{app:samp}) of
$m$ users among $n$ leads to \begin{equation*}
    \eps \leq \frac{8 \alpha L^2 \gamma^2}{m^2 \sigma^2}\frac{2m^2}{n^2}.
\end{equation*}
%Amplification by subsampling: we could do as in Altschuler 2.12
We conclude using composition over the $K$ rounds of the algorithm.
\end{proof}

\subsection{Fully Decentralized Private ADMM}

In the fully decentralized setting, the local privacy loss is the same as in
the previous section for the federated case. However, the threat model is
quite different. The privacy guarantees are with respect to the other users'
view, and each user will only observe information in time steps where he/she
participates.

We characterize the privacy loss by decomposing the problem
as follows. Starting from the LDP loss, we derive the privacy loss suffered by
a user $i$ when the $z$ variable is observed $m$ steps after
the contribution made by $i$. This is similar to the classic setting of
privacy
amplification by iteration where a model is only available after a given
number of steps (see \cref{app:iter}). Then, from the formula for a fixed
number of steps, we derive the privacy loss that accounts for the secrecy of
the path and the randomness of its length. This is done by using the weak
convexity
property of the Rényi divergence \citep{ampbyiteration} to weight each
scenario according to
the probability of the possible lengths. These probabilities can be
easily computed
as we consider a complete graph for the communication graph. We conclude the
proof by using composition over the maximum number of times $K_i$ any
user participates to the computation.

For convenience, we first restate the theorem, and then give the full proof.

\begin{reptheorem}{thm:privacy_dec}
    Assume that the loss function $f(\cdot,d)$ is $L$-Lipschitz for any
    local dataset $d$ and consider user-level DP.
    Let $\alpha >1, \sigma > 2L\gamma \sqrt{\alpha (\alpha -1)}$ and
    $K_i$ the
    maximum number of contribution of a user. Then
    Algorithm~\ref{algo:p2ppADMM} satisfies $(\alpha, \frac{8 \alpha K_i L^2
    \gamma^2\ln n} {\sigma^2 n})$-network RDP.
\end{reptheorem}
\begin{proof}
Here, a given user $j$ can only infer information about the other users
when it participates, by observing the
current value of the $z$ variable. Therefore, we can write the view
of user $j$ as:
\begin{equation*}
    \mathcal{O}_j(\mathcal{A}(D))=\big(z_{k_l(j)}\big)_{l=1}^{K_j},
\end{equation*}
where $k_l(j)$ is the time of $l$-th contribution of user $j$ to the
computation, and $K_j$ is the total number of times that $j$ contributed
during the execution of
algorithm. As we consider the complete graph, the
probability to visit $j$ at any step is exactly $1/n$. Hence, we
have closed forms for the probability that the random walk goes from a user
$i$ to another user $j$ in $m$ steps. Specifically, it follows the geometric
law of parameter $1/n$.

As an intermediate step of the proof, we thus express the privacy loss
induced by a user $i$ with respect to another user $j$ when there is exactly
$m$ steps after the participation of $i$ to reach $j$, meaning than $j$ will
only observe the variable $z_{k+m}$ if $i$ participated at time $k$,
and thus the contribution of $i$ has already been mixed with $m$
subsequent steps of the algorithm.

In this case, the privacy loss can be computed from the local privacy loss $\eps_{loc}$ in
\cref{eq:epsloc}, and the use of privacy
amplification by iteration in \cref{thm:amp_iter} where we have $s_1 = \eps_
{loc}$ and $s_{i>1} = 0$, and we set $a_i = \eps_{loc}/m$. This leads to
following bound:
\begin{equation*}
    \eps \leq \sum_{i=1}^m D_{\alpha}(\gau{\sigma^2}||\mathcal{N}(a_i,
    \sigma^2)) \leq  \frac{8\alpha L^2 \gamma^2}{\sigma^2 m}.
\end{equation*}

Now that we have a bound for a fixed number of steps between the two users, we
can compute the privacy loss for the random walk. Using the fact that the walk
remains private to the users, i.e. they do not observe the trajectory
of the walk except the times it passed through them, we can apply the weak
convexity of the Rényi divergence.
% to express the privacy loss during the
% random walk.

\begin{prop}[Weak convexity of Rényi divergence, \citealp{ampbyiteration}]
    \label{prop:convexity}
    Let $\mu_{1}, \ldots, \mu_{m}$ and $\nu_{1}, \ldots, \nu_{m}$ be probability
    distributions over some domain $\mathcal{Z}$ such that for all $i \in[m], D_
    {\alpha}\left(\mu_{i} \| \nu_{i}\right) \leq c /(\alpha-1)$ for some $c \in 
    (0,1]$. Let $\rho$ be a probability distribution over $[m]$ and denote by
    $\mu_{\rho}$ (resp. $\nu_{\rho}$) the probability distribution over $
    \mathcal{Z}$ obtained by sampling i from $\rho$ and then outputting a random
    sample from $\mu_{i}$ (resp. $\nu_{i}$). Then we have:
    \[
    D_{\alpha}\left(\mu_{\rho} \| \nu_{\rho}\right) \leq(1+c) \cdot \underset{i \sim \rho}{\E}\left[D_{\alpha}\left(\mu_{i} \| \nu_{i}\right)\right].
    \]
\end{prop}

Let us fix a contribution of user $i$ at some time $k(i)$. We apply this lemma
to $\rho$ the distribution of the number of steps before
reaching user $j$, which follows a geometric law of parameter $1/n$. This
gives:
\[ \begin{array}{lll}
    D_{\alpha}(z_j|| z'_j) & \leq & \sum_{k=1}^{K-k(i)} \frac{1}{n} (1 - 
    \frac{1}{n})^k     \frac{8\alpha L^2 \gamma^2}{2\sigma^2 k}
    \\
        & \leq & \frac{8 \alpha L^2 \gamma^2}{\sigma^2 n} \sum_{k=1}^{\infty} 
        \frac{(1- 1/n)^k}{k}\\
        & \leq & \frac{8 \alpha L^2 \gamma^2 \ln n}{\sigma^2 n}.
    \end{array}
    \]

Finally, we use composition to bound the total privacy loss. Each user
participates $K/n$ times in average, and this estimate concentrates as $K$
increases. For the sake of simplicity, we use $K_i = \bigo(K/n)$ as an
upper bound.
\end{proof}

% \begin{thm}[NDP of decentralized ADMM]
%     Let $\sigma \geq 2L\gamma \sqrt{\alpha (\alpha - 1)} $ Noting $K_u$ an upper bound of the number of contributions per user, \cref{algo:p2ppADMM} is $(\alpha. \eps)$-Renyi-Network DP for 
%     \begin{equation}
%         \eps \leq \bigo \left(\frac{8 K_i \alpha L^2 \gamma^2 \ln n}{\sigma^2 n}\right)
%     \end{equation}
% \end{thm}

% !TEX root = main.tex

\section{Privacy-Utility Trade-offs of Private ADMM Algorithms}
\label{app:priv_util}

Now, we amalgamate the privacy analysis of the three private ADMM algorithms 
(Appendix~\ref{app:privacy}) with the generic convergence analysis of
fixed-point iterations (Theorem~\ref{thm:utility_general}) to obtain the
privacy-utility trade-off for these three algorithms.

\subsection{Centralized Private ADMM}
Here, we present the detailed proof of Corollary~\ref{lem:priv_util_central}.
\begin{repcorollary}{lem:priv_util_central}
	Under the assumptions and notations of Theorem~\ref{thm:utility_general} and ~\ref{thm:privacy_centralized_cons}, and for number of iterations $K= \bigo\left( \log \left(\frac{L \gamma}{nD\left(1-\tau\right)} \left(\frac{p\alpha}{\varepsilon}\right)^{1/2}+\frac{p\alpha L^2 \gamma^2}{\varepsilon n^2 D\left(1-\tau\right)^3}\right)\right)$, Algorithm~\ref{algo:conADMM} achieves
	\begin{equation}
		\expect\left(\left\|u_{K+1}-u^*\right\|^{2}\right)  = \widetilde{\bigo}\left(\frac{L \gamma\sqrt{p\alpha}}{\sqrt{\varepsilon}n\left(1-\tau\right)} +\frac{p\alpha L^2 \gamma^2}{\varepsilon n^2 \left(1-\tau\right)^3}\right).
	\end{equation}
\end{repcorollary}
\begin{proof}
	We recall from Theorem~\ref{thm:utility_general} that
	\begin{align*}
	 \expect[\left\|u_{\step+1}-u^*\right\|^{2}]  \leqslant &\left(1 - \frac{q^2(1-\tau)}{8}\right)^{\step}\left(D+\frac{2(\sigma\sqrt{p} + \zeta)}{\sqrt{q}\left(1-\tau\right)}\right)+\frac{8(p\sigma^2+ \zeta^2)}{{q}^3(1-\tau)^3}+\frac{4}{\left(1-\tau\right)} \\
 %&\leq\left(1-\frac{q^2(1-\tau)}{8}\right)^{\step}\left(D+\frac{2(\sigma\sqrt{p}+ \zeta)}{\sqrt{q}\left(1-\tau\right)}\right)
%		+\left(\frac{2 (\sigma\sqrt{p}+ \zeta)}{\sqrt{q}\left(1-\tau\right)}+\frac{8(\sigma\sqrt{p}+ \zeta)^3}{{q}^{7/2}(1-\tau)^3}\right)\\
		&\leq\left(1-\frac{q^2(1-\tau)}{8}\right)^{\step}D
		+\left(\frac{8 (\sigma\sqrt{p}+ \zeta)}{\sqrt{q}\left(1-\tau\right)}+\frac{8(p\sigma^2+ \zeta^2)}{{q}^3(1-\tau)^3}\right)
	\end{align*}
	In case of centralized private ADMM, $\zeta=0$, $q=1$, and $\sigma^2 = \frac{8 K \alpha  L^2 \gamma^2}{\varepsilon n^2}$ (Theorem~\ref{thm:privacy_centralized_cons}). Thus, we obtain for $\step=K$ that
	\begin{align*}
	\expect\left(\left\|u_{K+1}-u^*\right\|^{2} \right)  
	&\leq\left(\frac{7+\tau}{8}\right)^{K}D
	+\left(\frac{8 \sqrt{p}}{\left(1-\tau\right)} \sqrt{ \frac{8 K \alpha  L^2 \gamma^2}{\varepsilon n^2}}+\frac{8p}{(1-\tau)^3}\left( \frac{8 K \alpha  L^2 \gamma^2}{\varepsilon n^2}\right)\right)\\
	&\leq\left(\frac{7+\tau}{8}\right)^{K}D
	+2\left(\frac{4 \sqrt{p}}{\left(1-\tau\right)} \sqrt{ \frac{8 \alpha  L^2 \gamma^2}{\varepsilon n^2}}+\frac{8p}{(1-\tau)^3}\left( \frac{8 \alpha  L^2 \gamma^2}{\varepsilon n^2}\right)\right) K\\
	&= \left(\frac{7+\tau}{8}\right)^{K}D
	+\bigo\left(\frac{L \gamma}{\left(1-\tau\right)n} \left(\frac{p\alpha}{\varepsilon}\right)^{1/2}+\frac{p\alpha L^2 \gamma^2}{\varepsilon n^2\left(1-\tau\right)^3}\right) K
	\end{align*}
	Now, if we consider $K$ such that
	\begin{align*}
		&\left(\frac{7+\tau}{8}\right)^{K}D = \bigo\left(\frac{L \gamma}{\left(1-\tau\right)n} \left(\frac{p\alpha}{\varepsilon}\right)^{1/2}+\frac{p\alpha L^2 \gamma^2}{\varepsilon n^2\left(1-\tau\right)^3}\right)\\ 
		\implies &K = \bigo\left( \log \left(\frac{L \gamma}{nD\left(1-\tau\right)} \left(\frac{p\alpha}{\varepsilon}\right)^{1/2}+\frac{p\alpha L^2 \gamma^2}{\varepsilon n^2 D\left(1-\tau\right)^3}\right)\right),
	\end{align*}
	we obtain
	\begin{align*}
	&\expect\left(\left\|u_{K+1}-u^*\right\|^{2} \right)  \\
	=&\bigo\left( \left(\frac{L \gamma}{n\left(1-\tau\right)} \left(\frac{p\alpha}{\varepsilon}\right)^{1/2}+\frac{p\alpha L^2 \gamma^2}{\varepsilon n^2 \left(1-\tau\right)^3}\right) \log \left(\frac{L \gamma}{nD\left(1-\tau\right)} \left(\frac{p\alpha}{\varepsilon}\right)^{1/2}+\frac{p\alpha L^2 \gamma^2}{\varepsilon n^2 D\left(1-\tau\right)^3}\right)\right) \\
	=&\widetilde{\bigo}\left(\frac{L \gamma\sqrt{p\alpha}}{\sqrt{\varepsilon}n\left(1-\tau\right)} +\frac{p\alpha L^2 \gamma^2}{\varepsilon n^2 \left(1-\tau\right)^3}\right)
	\end{align*}
\end{proof}
\subsection{Federated Private ADMM with Subsampling}
Here, we present the detailed proof of Corollary~\ref{lem:priv_util_fed}.
\begin{repcorollary}{lem:priv_util_fed}
	Under the assumptions and notations of Theorem~\ref{thm:utility_general} and ~\ref{thm:privacy_fed}, for number of iterations $K=\bigo\left(\log\left(\frac{\sqrt{p\alpha}L \gamma}{\sqrt{\varepsilon r}n D \left(1-\tau\right)} +\frac{p \alpha L^2 \gamma^2}{\varepsilon r^2 n^2 D \left(1-\tau\right)^3}\right)\right)$, and $m= r n$ for $r \in (0,1)$, Algorithm~\ref{algo:fedADMM} achieves
	\begin{equation}
	\expect\left(\left\|u_{K+1}-u^*\right\|^{2}\right)  = \widetilde{\bigo}\left(\frac{\sqrt{p\alpha}L \gamma}{\sqrt{\varepsilon r}n \left(1-\tau\right)} +\frac{p \alpha L^2 \gamma^2}{\varepsilon r^2 n^2 \left(1-\tau\right)^3}\right).
	\end{equation}
\end{repcorollary}
\begin{proof}
		In case of federated private ADMM, $\zeta=0$, $q=\frac{m}{n}$, and $\sigma^2 = \frac{16 K \alpha  L^2 \gamma^2}{\varepsilon n^2}$ (Theorem~\ref{thm:privacy_fed}). Thus, using Theorem~\ref{thm:utility_general}, we obtain for $\step=K$ that
		\begin{align*}
		\expect\left(\left\|u_{K+1}-u^*\right\|^{2} \right)  
		&\leq\left(1-\frac{m^2(1-\tau)}{8n^2}\right)^{K}D
		+\left(\frac{8 \sqrt{p}}{\left(1-\tau\right)} \sqrt{\frac{n}{m}} \sqrt{ \frac{16 K \alpha  L^2 \gamma^2}{\varepsilon n^2}}+\frac{8p}{(1-\tau)^3}\left( \frac{16 K \alpha  L^2 \gamma^2}{\varepsilon n^2}\right)\left(\frac{n}{m}\right)^{3}\right)\\
		&\leq\left(1-\frac{m^2(1-\tau)}{8n^2}\right)^{K}D
		+2\left(\frac{8 \sqrt{p}}{\left(1-\tau\right)} \sqrt{ \frac{8 \alpha  L^2 \gamma^2}{\varepsilon n m}}+\frac{8p}{(1-\tau)^3}\left( \frac{8 \alpha  L^2 \gamma^2}{\varepsilon n^2}\right)\frac{n}{m^{3}}\right) K\\
		&= \left(1-\frac{m^2(1-\tau)}{8n^2}\right)^{K} D
		+\bigo\left(\frac{L \gamma}{\left(1-\tau\right)} \left(\frac{p\alpha }{\varepsilon n m}\right)^{1/2}+\frac{L^2 \gamma^2 p\alpha }{\varepsilon\left(1-\tau\right)^3} \frac{{n}}{m^{3}}\right) K\\
		&= \left(1-\frac{m^2(1-\tau)}{8n^2}\right)^{K} D
		+\bigo\left(\frac{\sqrt{p\alpha}L \gamma}{\sqrt{\varepsilon r}n\left(1-\tau\right)} +\frac{p \alpha L^2 \gamma^2}{\varepsilon r^2 n^2 \left(1-\tau\right)^3}\right) K
		\end{align*}
		The last equality holds true when we choose $m= r n$, where $r \in (0,1/5]$ is a constant subsampling ratio.
		
		Now, if we consider $K=\bigo\left(\log\left(\frac{\sqrt{p\alpha}L \gamma}{\sqrt{\varepsilon r}n D \left(1-\tau\right)} +\frac{p \alpha L^2 \gamma^2}{\varepsilon r^2 n^2 D \left(1-\tau\right)^3}\right)\right)$, we obtain
		\begin{align*}
		&\expect\left(\left\|u_{K+1}-u^*\right\|^{2} \right)  \\
		=&\bigo\left(\left(\frac{\sqrt{p\alpha}L \gamma}{\sqrt{\varepsilon r}n \left(1-\tau\right)} +\frac{p \alpha L^2 \gamma^2}{\varepsilon r^2 n^2 \left(1-\tau\right)^3}\right) \log\left(\frac{\sqrt{p\alpha}L \gamma}{\sqrt{\varepsilon r}n D \left(1-\tau\right)} +\frac{p \alpha L^2 \gamma^2}{\varepsilon r^2 n^2 D \left(1-\tau\right)^3}\right)\right) \\
		=&\widetilde{\bigo}\left(\frac{\sqrt{p\alpha}L \gamma}{\sqrt{\varepsilon r}n \left(1-\tau\right)} +\frac{p \alpha L^2 \gamma^2}{\varepsilon r^2 n^2 \left(1-\tau\right)^3}\right)
		\end{align*}
\end{proof}
\subsection{Fully Decentralized Private ADMM}
Here, we present the detailed proof of Corollary~\ref{lem:priv_util_dec}.
\begin{repcorollary}{lem:priv_util_dec}
	Under the assumptions and notations of Theorem~\ref{thm:utility_general} and ~\ref{thm:privacy_dec}, and for number of iterations $K=\bigo\left(\log\left(\frac{L \gamma}{D \left(1-\tau\right)} \left(\frac{p\alpha \ln n}{\varepsilon n}\right)^{1/2}+\frac{L^2 \gamma^2}{D \left(1-\tau\right)^3} \left(\frac{p\alpha \ln n}{\varepsilon n}\right)\right)\right)$, Algorithm~\ref{algo:p2ppADMM} achieves
	\begin{equation}
	\expect\left(\left\|u_{K+1}-u^*\right\|^{2}\right)  = \widetilde{\bigo}\left(\frac{\sqrt{p \alpha}L \gamma}{\sqrt{\varepsilon n}\left(1-\tau\right)} +\frac{p\alpha  L^2 \gamma^2}{\varepsilon n \left(1-\tau\right)^3} \right).
	\end{equation}
\end{repcorollary}
\begin{proof}
	In case of decentralized private ADMM, $\zeta=0$, $q=\frac{1}{n}$, and $\sigma^2 = \frac{8 K_i \alpha L^2 \gamma^2 \ln n}{\sigma^2 n} = \frac{8 K \alpha L^2 \gamma^2 \ln n}{\sigma^2 n^2}$ (Theorem~\ref{thm:privacy_fed}). Thus, using Theorem~\ref{thm:utility_general}, we obtain for $\step=K$ that
	\begin{align*}
	\expect\left(\left\|u_{K+1}-u^*\right\|^{2} \right)  
	&\leq\left(1-\frac{q^2(1-\tau)}{8}\right)^{\step}D
	+\left(\frac{8 (\sigma\sqrt{p}+ \zeta)}{\sqrt{q}\left(1-\tau\right)}+\frac{8(p\sigma^2+ \zeta^2)}{{q}^3(1-\tau)^3}\right)\\
	&\leq\left(1-\frac{1-\tau}{8n^2}\right)^{K}D
	+\left(\frac{8 \sqrt{p}}{\left(1-\tau\right)} \sqrt{n} \sqrt{ \frac{8 K \alpha  L^2 \gamma^2 \ln n}{\varepsilon n^2}}+\frac{8 p}{(1-\tau)^3}\left( \frac{8 K \alpha  L^2 \gamma^2 \ln n}{\varepsilon n^2}\right)n^{3}\right)\\
	&\leq\left(1-\frac{1-\tau}{8n^2}\right)^{K}D
	+2\left(\frac{8 \sqrt{p}}{\left(1-\tau\right)} \sqrt{ \frac{8 \alpha  L^2 \gamma^2 \ln n}{\varepsilon n}}+\frac{8 p}{(1-\tau)^3}\left( \frac{8 \alpha  L^2 \gamma^2 \ln n}{\varepsilon n}\right)\right) K\\
	&= \left(1-\frac{1-\tau}{8n^2}\right)^{K}D
	+\bigo\left(\frac{L \gamma}{\left(1-\tau\right)} \left(\frac{p\alpha \ln n}{\varepsilon n}\right)^{1/2}+\frac{L^2 \gamma^2}{\left(1-\tau\right)^3} \left(\frac{p\alpha \ln n}{\varepsilon n}\right)\right) K
	\end{align*}
	Now, if we consider $K=\bigo\left(\log\left(\frac{L \gamma}{D \left(1-\tau\right)} \left(\frac{p\alpha \ln n}{\varepsilon n}\right)^{1/2}+\frac{L^2 \gamma^2}{D \left(1-\tau\right)^3} \left(\frac{p\alpha \ln n}{\varepsilon n}\right)\right)\right)$, we obtain
	\begin{align*}
	&\expect\left(\left\|u_{K+1}-u^*\right\|^{2} \right)  \\
	=&\bigo\left(\left(\frac{L \gamma}{\left(1-\tau\right)} \left(\frac{p\alpha \ln n}{\varepsilon n}\right)^{1/2}+\frac{L^2 \gamma^2}{\left(1-\tau\right)^3} \left(\frac{p\alpha \ln n}{\varepsilon n}\right)\right)\log\left(\frac{L \gamma}{D \left(1-\tau\right)} \left(\frac{p\alpha \ln n}{\varepsilon n}\right)^{1/2}+\frac{L^2 \gamma^2}{D \left(1-\tau\right)^3} \left(\frac{p\alpha \ln n}{\varepsilon n}\right)\right)\right) \\
	=&\widetilde{\bigo}\left(\frac{\sqrt{p \alpha}L \gamma}{\sqrt{\varepsilon n}\left(1-\tau\right)} +\frac{p\alpha  L^2 \gamma^2}{\varepsilon n \left(1-\tau\right)^3} \right)
	\end{align*}
\end{proof}

% !TEX root = main.tex

\section{Numerical Experiments}
\label{app:expes}

In this section, we illustrate the performance of our private
ADMM algorithms on
the classic
Lasso problem, which is widely used to learn sparse
solutions to regression problems with many features.
Lasso aims to solve the following problem:
\begin{mini}
{x\in\R^p}{
\frac{1}{2n} \norm{A x-b}^2 + \kappa \norm{x}_1,}{}{}
\end{mini}    
where the dataset $D=(A,b)$ consists of $n$ labeled data points in $p$
dimensions, represented by a
matrix $A \in \R^{n\times p}$ and a vector of regression targets $b \in
\R^n$. The previous objective can be rewritten as a consensus problem of the
form \eqref{eq:admm_consensus}, with
the same notations:
%\[L(x,z, D) = \frac{1}{2n}\sum_{i=1}^n \norm{\transp{(D_x^i)} x_i - D_y^i }^2 + \alpha \norm{z}_1  \]
\begin{mini}
    {x\in\mathbb{R}^{np}, z\in\mathbb{R}^p}{
    \label{eq:lasso}\frac{1}
    {2n}\sum_{i=1}^n \big(\transp{(A^i)} x-b^i\big)^2 + \kappa \norm{z}_1}{}{}
    \addConstraint{x- I_{n(p\times p)}z}{= 0, \quad}{}
\end{mini}    
where $A^i$ is the $i$-th row of $A$ and $b^i$ the $i$-th coordinate of $b$.

The corresponding ADMM updates take simple forms \citep{admmbook}. The $z$-update,
defined as $
\proxi{\gamma \kappa \norm{\cdot}_1}{\hat{z}}$, corresponds to the soft
thresholding function with parameter $\gamma\kappa$.
For the $x$-update, we have $x_i = \proxi{\gamma/2n (\transp{(A^i)} \cdot -
b^i)^2}{2z-u_i}$. This also gives a closed-form update:
\[
x_i = (A^i \transp{(A^i)} + (2n/\gamma) I)^{-1}(b^i A^i + (2n/\gamma) 
(2z-u_i)).
\]
Note that the matrix to invert is a rank-one perturbation of the
identity, so the inverse can be computed via the Sherman Morrison formula.
As usually done in privacy-preserving machine learning, we ensure a tight
evaluation of the sensitivity by using clipping.
The full algorithm is given in Algorithm~\ref{algo:lasso}.

We generate synthetic data by drawing $A$ as random vectors from the
$p$-dimensional unit
sphere, and draw the ground-truth model $x$ from a
uniform distribution with support of size $8$. Labels are then obtained by
taking $b = A x + \eta$ where $\eta\sim\mathcal{N}(0,0.01)$. We use
$n=1000$ and $p=64$. 

As reference, we solve the non-private problem with \texttt{scikit-learn}, and
we use the best regularization parameter $\kappa$ obtained by
cross-validation. For
comparison purposes, we also implement (proximal)
DP-SGD where noise is added to the gradients of the smooth part.
For both approaches, we tune the step size and clipping threshold using grid
search. For ADMM, we also tune the $\gamma$ parameter. For
simplicity, we tune these parameters on the smallest privacy budget and use
the obtained parameters for all budgets, even if slight improvements could be
achieved by tuning these parameters for each setting.
We use the same number
of iterations $K$ for both algorithms.
% , even though better trade-offs might be
% possible to achieve by tuning this parameter as well.

\begin{algorithm}[t]
    \SetKwComment{Comment}{$\triangleright$ }{}
    \DontPrintSemicolon
    \KwIn{initial vector $u^0$, step size $\lambda \in (0, 1]$, privacy noise
    variance $\sigma^2\geq 0$, $\gamma >0$, clipping threshold $C$ }
    \For{$k=0$ to $K-1$}{
        $\hat{z}_{k+1}=\frac{1}{n} \sum_{i=1}^{n}u_{k,i} $ \;
        $z_{k+1}= S_{\gamma \kappa} \left(\hat{z}_{k+1}\right)$\;
        \For{$i=1$ to $n$}{ 
            $x_{k+1,i} = (A^i \transp{(A^i)} + (2n/\gamma) I)^{-1}(b^i A^i + (2n/\gamma) (2z_k-u_{k,i}))$\;
            $u_{k+1,i}=u_{k,i}+2 \lambda\big(\operatorname*{Clip}(x_{k+1,i}-z_
            {k+1}, C)+\frac{1}
            {2} \eta_{k+1,i}\big)\text{ with } \eta_{k+1,i}\sim\gau{ \sigma^2
        \mathbb{I}_p}$ \;
        }
    }
    \Return{$z_{K}$}    
    \caption{Centralized private ADMM for Lasso}
    \label{algo:lasso}
\end{algorithm}

% We tune the parameters on the non-private task, using grid search. Then,
We report the objective function value on the test set at the end of the
training for several privacy budgets. Privacy budgets are converted to $
(\eps, \delta)$-DP for the sake of comparison with
existing methods. The conversion to Rényi DP is done numerically. We set
$\delta=10^{-6}$ in all cases.

The resulting privacy-utility trade-offs for the federated setting with
central DP are
shown in \cref{fig:expe}, where each user has a single datapoint and users are
sampled uniformly with a 10\% probability. We see that private ADMM performs
especially well in high privacy regimes. Note that the $y$ axis is in
logscale, so the improvement over DP-SGD is significant. This could be
explained by the fast convergence of ADMM at the beginning of training, and
by the robustness of its updates. The two curves go flat for low privacy
budgets: this is simply because these regimes would
require more training steps and smaller step-sizes to converge to more
precise solutions.

\begin{figure}
    \centering
    \includegraphics[width=.6\textwidth]{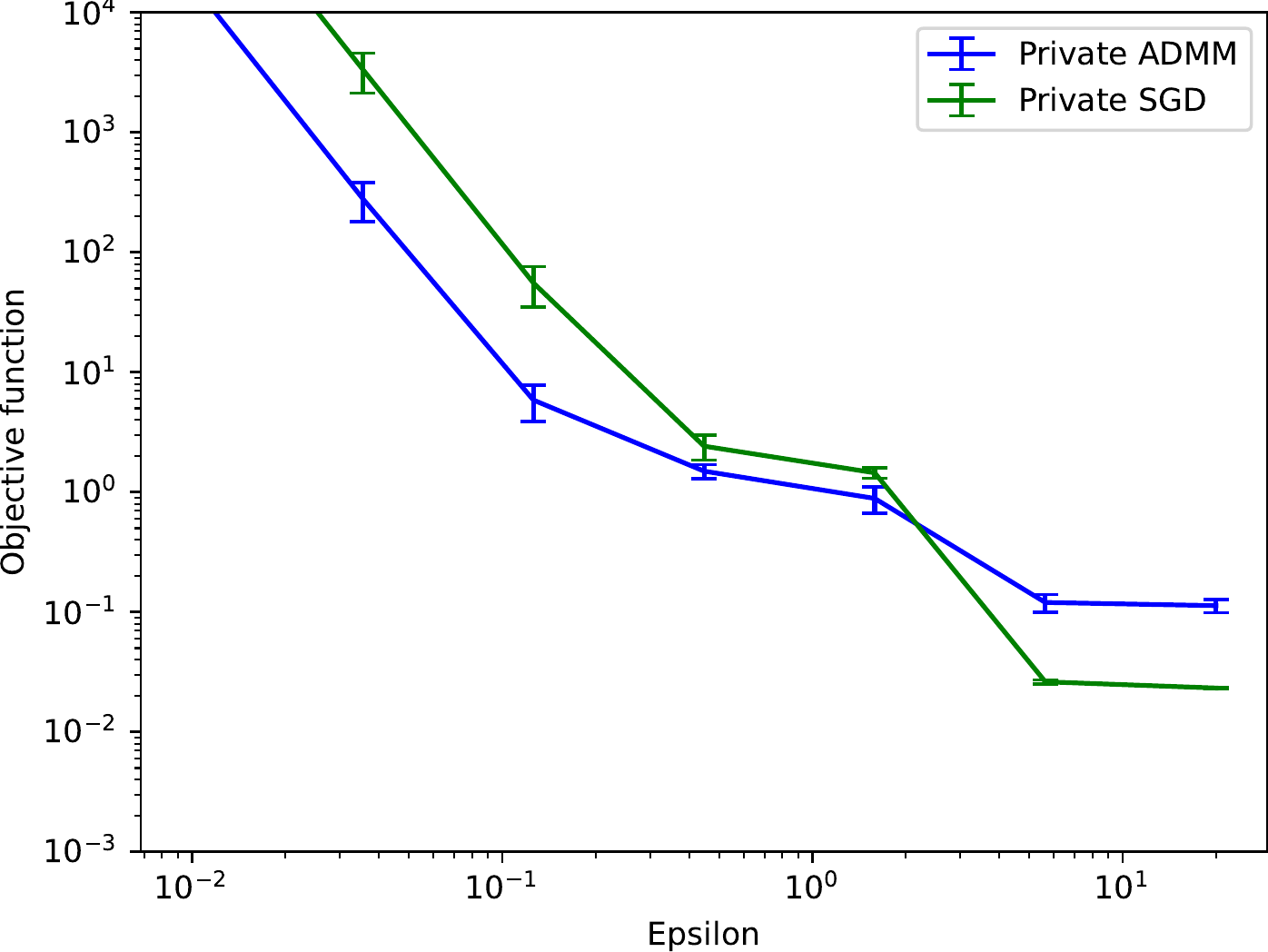}
    \caption{Comparison of DP-SGD and our DP-ADMM algorithm for the Lasso
    problem on
    synthetic data ($n=1000$, $p=64$). The
    same regularizer parameter is used. We show here results for the federated
    setting, with a user sampling probability of $10\%$. Each setting is run $10$ times, and we report average and standard deviation\label{fig:expe}}
\end{figure}

The code is available at \url{https://github.com/totilas/padadmm}.

%%%%%%%%%%%%%%%%%%%%%%%%%%%%%%%%%%%%%%%%%%%%%%%%%%%%%%%%%%%%%%%%%%%%%%%%%%%%%%%

\end{document}